\documentclass[10pt]{article}
\usepackage[preprint]{tmlr}
\usepackage[T1]{fontenc}
\usepackage[utf8]{inputenc}
\usepackage{amsmath,amssymb,mathtools,amsthm}
\usepackage{algorithm,algpseudocode}
\usepackage{booktabs,multirow,array}
\usepackage{graphicx,xcolor,hyperref,url}
\usepackage{microtype}

\DeclareMathOperator*{\argmin}{arg\,min}
\newtheorem{theorem}{Theorem}[section]
\newtheorem{proposition}[theorem]{Proposition}
\newtheorem{corollary}[theorem]{Corollary}
\newtheorem{lemma}[theorem]{Lemma}
\newtheorem{definition}{Definition}[section]
\newtheorem{example}{Example}[section]
\newtheorem{remark}{Remark}
\title{Credal and Interval Deep Evidential Classifications}
\author{\name Michele Caprio \email Michele.Caprio@warwick.ac.uk \\
       \addr Department of Computer Science\\
       University of Warwick\\
       Coventry, UK CV4 7AL
       \AND
       \name Shireen Manchingal \email shireen.manchingal@oxdynamics.com \\
       \addr Oxford Dynamics\\
       Oxford, UK OX11 0QX
       \AND
       \name Fabio Cuzzolin \email fabio.cuzzolin@brookes.ac.uk \\
       \addr School of Engineering, Computing and Mathematics\\
       Oxford Brookes University\\
       Oxford, UK OX3 0BP}

\begin{document}
\maketitle
\begin{abstract}
Credal Bayesian Deep Learning (CBDL) brings Bayesian sensitivity analysis to deep learning by propagating sets of priors and likelihoods to sets of posterior predictives, called  predictive credal sets. Its implementation requires specifying these priors and likelihoods, and approximating multiple posterior and posterior predictive distributions, which can be computationally demanding. In this paper, we introduce Credal and Interval Deep Evidential Classifications (CDEC and IDEC, respectively), which construct predictive credal sets for classification using evidential networks and avoid integration over network weights at deployment. CDEC forms a convex hull of evidential posterior predictives, and retains the predictive structure of a single prior/multiple likelihoods restriction of CBDL. IDEC constructs a predictive credal set from a single evidential predictor by allowing controlled perturbations of its predicted class probabilities.  We develop a method for selecting the perturbation size using predictions from a fitted CDEC or CBDL model on separate calibration data. This controls, on average, how much IDEC can underestimate decision risks assessed by that model. Our experiments show competitive classification accuracy on clean data and lower costs for computing class probabilities relative to a restricted CBDL baseline. They also show that closer approximation to the reference can require larger IDEC prediction regions.\footnote{OpenAI's ChatGPT Astra was used to assist with proofreading, improving the clarity and presentation of the manuscript, refining the experimental design and analysis, and preparing figures.}
\end{abstract}

\section{Introduction}

Uncertainty quantification is important for decision making with machine
learning models, particularly when incorrect predictions may have substantial
consequences
\citep{senge_2014_ReliableClassificationLearning,varshney2016engineering}.
Credal Bayesian Deep Learning (CBDL) \citep{ibnn} brings Bayesian
sensitivity analysis to deep learning by considering alternative priors
and likelihoods and propagating them to sets of predictive distributions; this allows uncertainty about the underlying probabilistic model to
inform predictions and decisions. Its reported benefits in uncertainty
quantification and downstream tasks make CBDL a promising approach.
Its implementation, however, can be computationally demanding, as it
requires Bayesian inference for multiple combinations of priors and
likelihoods, followed by integration over network weights---approximated via variational inference---to obtain
predictions.

These computational questions motivate Self-Supervised Laplace
Approximation \citep{rodemann2026ssla}, which approximates posterior
predictive distributions directly and supports sensitivity analysis
across prior choices, with a focus on regression. The present paper continues this line of research in classification,
building on the predictive construction of CBDL with a single
specified prior and several plausible likelihood models. We introduce Credal and
Interval Deep Evidential Classifications (CDEC and IDEC, respectively),
which use evidential networks to obtain predictive class probabilities
without integration over network weights at deployment.

CDEC combines the predictive distributions of several evidential
networks into a closed and convex set of probabilities, called a credal set. We use this set to quantify aleatoric (irreducible) and
epistemic (reducible) uncertainty (AU and EU, respectively), and to determine whether to return
a collection of labels or abstain. 
Under every predictive distribution in the credal set, the probability
that the returned collection contains the correct label is at least
a prescribed level chosen by the user. 
We establish the mathematical relationship between this construction
and CBDL with a single specified prior and several plausible likelihood
models, and assess the predictive performance and computational costs
of the separately fitted models through matched experiments.

IDEC simplifies the construction by using a single evidential
predictor and allowing controlled perturbations of its predicted
class probabilities. The parameter $d$ determines the size of these
perturbations and therefore affects both the uncertainty measures
and the prediction regions. 
We develop a calibration procedure for selecting $d$ using a fitted
CDEC or CBDL model as a reference; we derive efficient formulas for comparing their predictive
sets and use separate calibration inputs to select $d$ according to
a prescribed approximation tolerance. Through conformal risk control
\citep{angelopoulos2024crc}, we obtain a guarantee that limits, on
average, how much IDEC can underestimate decision risks assessed by
the reference. The reference is constructed and evaluated offline;
deployment uses only the single evidential predictor and the calibrated
perturbation size.

\paragraph{Imprecise Probabilistic Machine Learning and Related Approaches.}
The constructions considered here belong to the broader framework of
imprecise probabilities, which provides methods for reasoning when
the available information supports several plausible probability
distributions \citep{walley,augustin}. One common representation is
a credal set: a collection of distributions retained under the
available information and modelling assumptions. For any event, the lower and upper envelopes of the range of probabilities assigned by these distributions are referred to as lower and upper probabilities, respectively. They allow uncertainty about
the probability assignment itself to be represented explicitly.

In machine learning, limited data and unresolved modelling choices
can leave several predictive distributions plausible. Retaining
these alternatives allows us to assess how predictions and decisions
depend on the underlying probabilistic model \citep{eyke}. In
classification, this can inform whether to return a single label,
provide several candidate labels, or abstain. These considerations
motivate imprecise probabilistic machine learning (IPML), which
studies how to learn such uncertainty representations, quantify
the uncertainty they express, and use them for prediction and
decision making \citep{eyke,conformal-regions,chau2025integral}. CBDL and the methods developed in this paper
contribute to this research programme.

One question within IPML is how to summarise the uncertainty
represented by a credal set. Lower and upper entropy provide an
established basis for distinguishing AU and EU
\citep{abellan,eyke}. Proper scoring rules, which encourage truthful
probability assessments, provide another foundation for constructing
measures of these uncertainties \citep{hofman2024proper}.

A related line of work develops neural networks that predict
credal sets. Credal Deep Ensembles (CreDE) combine networks that
learn lower and upper probabilities for each class
\citep{wang2024crede}. CREDIT addresses the computational cost of
ensemble prediction by distilling an ensemble into a single
network that outputs class probability bounds
\citep{wang2026credit}. These methods connect the representation
of predictive uncertainty with the computational requirements
of deployment.

Conformal prediction also admits an interpretation through
credal sets. \citet{CELLA20221} show how conformal scores can
be used to construct upper probabilities and an associated
credal set. For credal sets obtained through this construction,
\citet{conformal-regions} identify the conformal prediction
regions as imprecise highest density regions, introduced in Section \ref{sec:ihdr}.

Calibration offers a further approach to determining the size
of a predictive credal set. Conformalized Distillation for Credal
Inference (CD-CI) uses predictions from a larger reference model
on separate calibration inputs to construct credal sets around
the predictions of a smaller model
\citep{huang2025distilling}. Its coverage guarantee concerns
inclusion of the reference predictive distribution. Other
conformal approaches use calibration data labelled by probability
distributions over classes. These support guarantees that a
predicted credal set contains the distribution associated with
a new input \citep{javanmardi2024credal,caprio2025conformalized}, or allow the coverage
of label prediction sets constructed from credal models to be
calibrated \citep{javanmardi2026optimal}.

Our work studies the construction of predictive credal sets from
evidential networks and the selection of their imprecision.
IDEC draws on intervals of measures \citep{lorraine,coolen} to construct
a credal set around a single evidential prediction. We develop
an efficiently computable criterion for approximating an entire
reference credal set and use conformal risk control
\citep{angelopoulos2024crc} to select the interval width.
The resulting guarantee controls, on average, how much IDEC
can underestimate decision risks assessed by the reference.

\paragraph{Contributions and structure of the paper.}
Our contributions are as follows.
\begin{itemize}
    \item \textbf{CDEC and its relationship with CBDL.}
    We establish the mathematical relationship between the predictive
    construction of CDEC and CBDL with a single specified prior and
    several plausible likelihood models. We also give conditions under which small changes in the class
probabilities predicted by the evidential networks lead to small
changes in the uncertainty measures and leave the resulting
classification or abstention decision unchanged. We formulate an extension of
    evidential training that incorporates corrupted inputs.

    \item \textbf{The interpretation and selection of interval width.}
    For IDEC, we distinguish the maximum perturbation $d^\star$ preserving
a fixed prediction region's probability guarantee from the
width $d$ used to represent predictive imprecision. We identify
limitations of alternative choices of $d$ and derive explicit formulas for measuring how closely an IDEC
interval approximates the predictive credal set of a fitted
CDEC or CBDL model.

    \item \textbf{Calibration of IDEC against a reference model.}
    We develop a procedure for selecting $d$ using separate
    calibration inputs and a fitted CDEC or CBDL model as a
    reference. By applying conformal risk control, we bound
    how much IDEC can underestimate decision risks assessed by
    this reference.
    A second calibration rule bounds the expected
    underestimation with high probability over the calibration
    sample, under given sampling assumptions.

    \item \textbf{An independent empirical evaluation.}
    A study on MNIST compares the proposed methods with
    a restricted CBDL implementation, CD-CI, CREDIT and CreDE.
    We use separate data for fitting, tuning, calibration and
    testing. The evaluation measures classification accuracy,
    approximation of the reference, prediction region size,
    uncertainty ranking and computational costs. It examines
    the trade-off between closer reference approximation and
    larger prediction regions, together with sensitivity to
    the calibration population and image corruption.
\end{itemize}

Section~\ref{back} gives the background. Section~\ref{unc's} introduces predictive
credal sets and uncertainty functionals. Sections~\ref{ien} and~\ref{ien-simpl}
present CDEC and IDEC. Section~\ref{sec:experiments} reports the independent MNIST study,
and Section~\ref{concl} discusses their implications and limitations. Proofs and
implementation details are provided in the appendices.

\section{Background}\label{back}

We briefly review the probabilistic foundations of CDEC and IDEC,
starting with the distributions used in evidential classification.

\subsection{Dirichlet and Categorical Distributions}\label{dir_and_cat}

In this section, we introduce two distributions that are fundamental to the
procedures presented in this paper. In a $k$-class classification problem,
a Categorical distribution $\operatorname{Cat}(\pi)$ is parameterised by a
probability vector $\pi\in\Delta^{k-1}$, where $\Delta^{k-1}$ denotes the $(k-1)$-dimensional simplex. Its probability mass function is
$P(\{j\})=\pi_j$, for $j\in\mathcal Y=\{1,\ldots,k\}$.

Adopting a Bayesian model for this vector, we specify a prior distribution for
$\pi$. The conjugate prior is a Dirichlet distribution
$\operatorname{Dir}(\alpha)$ \citep{hoff}, with $\alpha_j>0$ for every $j$.
Its density on the simplex is
\[
f(\pi)=\frac1{B(\alpha)}\prod_{j=1}^k\pi_j^{\alpha_j-1},
\qquad \mathbb E[\pi_j]=\frac{\alpha_j}{\sum_l\alpha_l},
\]
where $B$ is the multivariate Beta function. The parameters are often interpreted
as prior pseudo-counts. If independent categorical observations supply class
counts $c_j$, the posterior is $\operatorname{Dir}(\alpha+c)$ and the posterior
predictive distribution is $\operatorname{Cat}(\pi^\prime)$ with
\[
\pi_j^\prime=\frac{\alpha_j+c_j}{\sum_l(\alpha_l+c_l)}.
\]
The next section replaces literal counts by input-dependent virtual evidence.
The fitted evidential parameters in that construction are at least one, although
strict positivity is the general condition defining a Dirichlet distribution.

\begin{remark}\label{dir_sec_ord}
A Dirichlet distribution is a distribution over Categorical probability vectors,
and is therefore a second-order distribution in this sense. Its Categorical
predictive mean is a first-order distribution on labels. Our credal uncertainty
functionals operate on sets of those first-order predictives. The resulting Dirichlet distribution represents uncertainty over
class probabilities; interpreting its concentration parameter
as a calibrated measure of uncertainty requires separate
statistical justification \citep{ulmer,mira}.
\end{remark}

\subsection{Posterior Predictive via Normalizing Flows}\label{post_nf}

We use the Posterior Network construction of \citet{charpentier} to obtain the
Categorical predictive distribution $\operatorname{Cat}(\pi^\prime)$. In the standard supervised learning setting, let
$D=\{(x_i,y_i)\}_{i=1}^n$ be a realisation of independent and
identically distributed pairs $(X_i,Y_i)$ from an unknown
joint distribution on $\mathcal X\times\mathcal Y$, where
$\mathcal Y=\{1,\ldots,k\}$. An encoder $f_\theta$ maps a new input
$\tilde x$ to $\tilde z=f_\theta(\tilde x)$. A class-specific normalising flow \citep{rezende} models the
distribution of latent representations for inputs with label $j$.
Its density $q(\tilde z\mid y=j,\phi)$ describes the local
concentration of probability around the encoded representation
$\tilde z$ of $\tilde x$, where $\phi$ denotes the fitted flow
parameters. If $n_j$ is the number of training observations
with label $j$, the corresponding virtual count is
\[
n_j^{\mathrm{virt}}(\tilde x)
=n_j \cdot q(\tilde z\mid y=j,\phi).
\]
The parameters $(\theta,\phi)$ are fitted through the evidential objective in
\citet[Equation (7)]{charpentier}.

This construction builds on the conjugate Bayesian model described
in Section~\ref{dir_and_cat}. To recall the conjugate Bayesian update, consider observations
that are conditionally independent and identically distributed
according to $\operatorname{Cat}(\pi)$, given a common
probability vector $\pi$. A Dirichlet prior is then updated
by adding the observed class counts to its parameters. The posterior predictive distribution is then
obtained by integrating the Categorical probabilities with
respect to the resulting Dirichlet posterior. Posterior Networks
 use this conjugate structure with learned
virtual counts that depend on the input being classified.

At a fixed input $\tilde x$, write
$c_j=n_j^{\mathrm{virt}}(\tilde x)$ and let
$c=(c_1,\ldots,c_k)$. We specify the common prior
$Q_0=\operatorname{Dir}(\mathbf1_k)$ and the evidential
likelihood factor
\begin{equation*}
L_{\tilde x}(\pi)=\prod_{j=1}^k\pi_j^{c_j}.
\end{equation*}
This factor has the same dependence on $\pi$ as a Categorical
likelihood with class counts $c_j$. Here, the virtual counts
are learned evidence weights: they are nonnegative, may be
noninteger, and vary with $\tilde x$.

Following the form of Bayes' rule, we combine this evidence
with the prior:
\[
Q_{\tilde x}(\mathrm d\pi)
=
\frac{L_{\tilde x}(\pi)\,Q_0(\mathrm d\pi)}
{\int_{\Delta^{k-1}}L_{\tilde x}(u)\,Q_0(\mathrm du)}.
\]
The conjugate calculation gives the evidential posterior
\[
Q_{\tilde x}=\operatorname{Dir}(\mathbf1_k+c).
\]
Thus, the virtual counts play the same algebraic role in this
update as observed counts do in the classical
Dirichlet-Categorical model.
For the new label, the Categorical model gives
$p(\tilde Y=j\mid\pi,\tilde x)=\pi_j$. Its evidential posterior
predictive distribution is obtained by integrating over the
uncertain class probabilities:
\begin{equation}\label{eq:postnet-mean}
\pi_j^\prime(\tilde x)
=
\int_{\Delta^{k-1}}
p(\tilde Y=j\mid\pi,\tilde x)\,
Q_{\tilde x}(\mathrm d\pi)
=
\frac{1+n_j^{\mathrm{virt}}(\tilde x)}
{\sum_{l=1}^k(1+n_l^{\mathrm{virt}}(\tilde x))}.
\end{equation}
These probabilities are strictly positive. The encoder and
density parameters are fitted using the training data and
determine the evidence supplied at each input. Prediction
then consists of evaluating these fitted models and using
the analytic posterior predictive expression above. Figure~\ref{fig-edl-proc} summarises the predictive
construction.
\begin{figure}[h!]
\centering
\includegraphics[width=.85\linewidth]{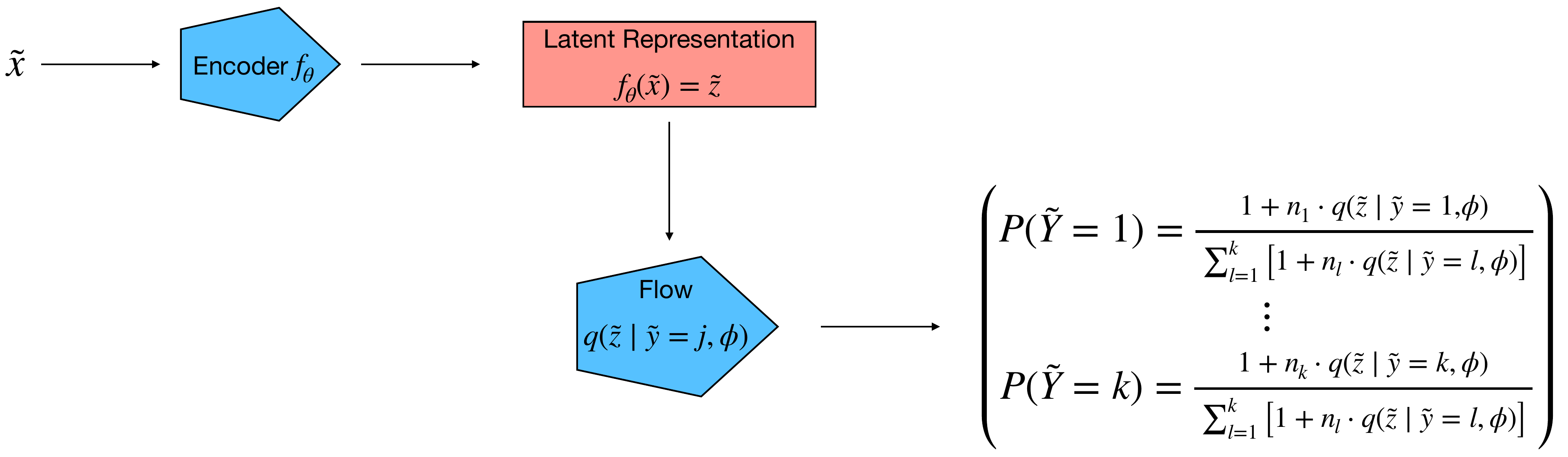}
\caption{The evidential predictive construction of \citet{charpentier}.}
\label{fig-edl-proc}
\end{figure}

\section{Credal Uncertainty Modelling}\label{unc's}
We now extend the evidential construction to allow several
plausible likelihood specifications under a common prior.
We combine the resulting posterior predictive distributions
into a credal set, which provides the basis for the uncertainty
measures and prediction regions used by CDEC.

\subsection{Likelihood Specification and Predictive Credal Sets}\label{lik-spec}

Consider $S$ fitted evidential models, indexed by
$s\in\{1,\ldots,S\}$, which may differ in their encoders,
density models or architectures. At an input $\tilde x$,
model $s$ supplies a virtual count vector $c_s(\tilde x)$
and an evidence factor $L_{s,\tilde x}$. Updating the common
prior $\operatorname{Dir}(\mathbf1_k)$ gives
\[
Q_{s,\tilde x}
=\operatorname{Dir}(\mathbf1_k+c_s(\tilde x)).
\]
We define the predictive credal set as the convex hull of the
Categorical predictive distributions induced by
$Q_{1,\tilde x},\ldots,Q_{S,\tilde x}$.
\begin{definition}[Credal Set \citep{levi2}]\label{credal-set}
A nonempty closed convex set of probability measures is a credal set.
On the finite label space, it is a compact subset of the probability simplex.
It is finitely generated if it is the convex hull of finitely many probability
measures. Its extreme elements are those that cannot be expressed as a
nontrivial convex combination of other elements of the set.
\end{definition}

For $S\geq1$ fitted evidential models, let $P_s=\operatorname{Cat}(\pi_s^\prime)$,
where $\pi_s^\prime$ is given by \eqref{eq:postnet-mean}. Define
\begin{equation}\label{eq:predictive-credal}
\mathcal P_{\mathrm{pred}}(\tilde x)
=\operatorname{Conv}\{P_1,\ldots,P_S\}.
\end{equation}
A generator can be redundant, so the set of extreme elements is a subset of the
supplied generators. Random initialisations produce different fitted parameter
values within the same model specification. In Section~\ref{sec:experiments}, we consider CBDL with a single
common Gaussian prior over network parameters and three
alternative likelihood models, defined by LeakyReLU,
hyperbolic tangent and SiLU activation functions. CDEC uses
the same activation functions and encoder dimensions.

Recall that a predictive kernel describes the distribution of a new output
for each input and parameter value. At a fixed input $x$,
$K_x(\theta,A)$ is the probability assigned to the event
$\{Y\in A\}$ by the model with parameter $\theta$.
For each $\theta$, $K_x(\theta,\cdot)$ is a probability
distribution on $\mathcal Y$, and for each $A\subseteq\mathcal Y$,
the map $\theta\mapsto K_x(\theta,A)$ is measurable.
For the Categorical model considered above, $\theta=\pi$ and
\[
K_x(\pi,A)=\sum_{j\in A}\pi_j.
\]
Integrating the kernel against a posterior distribution $Q$
gives the corresponding posterior predictive distribution.

\begin{proposition}[Posterior-to-Predictive Correspondence]\label{prop:affine}
Let $Q_1,\ldots,Q_S$ be posterior distributions obtained from one prior and a
finite likelihood family. For a common predictive kernel $K_x$, define
$(T_xQ)(A)=\int K_x(\theta,A) Q(\mathrm d\theta)$. Then
\[
T_x \left(\operatorname{Conv}\{Q_1,\ldots,Q_S\}\right)
=\operatorname{Conv}\{T_xQ_1,\ldots,T_xQ_S\}.
\]
For the evidential construction with $K(\pi,\{j\})=\pi_j$, this predictive image
is exactly \eqref{eq:predictive-credal}.
\end{proposition}
The identity follows from linearity of integration. The same linearity argument applies when each model has its
own predictive kernel $K_{s,x}$. We represent the parameter
value together with the model it belongs to, using the pair
$(s,\theta)$. On these pairs, define a single kernel by
$\widetilde K_x((s,\theta),A)\equiv K_{s,x}(\theta,A)$. 
We represent $Q_s$ on this enlarged space by keeping the
model index fixed at $s$ and letting $\theta$ have
distribution $Q_s$. Each posterior is therefore evaluated
using the predictive kernel of its own model.

There is also a Bayesian relationship between mixtures of
likelihood factors and mixtures of their posteriors. Suppose
the posteriors $Q_s$ are obtained from a common prior $\Pi_0$
using the complete likelihood or evidence factors $L_s$,
with normalising constants
$Z_s=\int L_s\,\mathrm d\Pi_0\in(0,\infty)$.
For nonnegative weights $\lambda_s$ summing to one, updating
the prior with $L_\lambda=\sum_s\lambda_sL_s$ gives
\[
Q_\lambda
=
\sum_s
\frac{\lambda_s Z_s}{\sum_t\lambda_t Z_t}\,Q_s.
\]
Thus, Bayes' rule adjusts the mixture weights according to
the normalising constants. As the likelihood weights vary,
the resulting posteriors range over the entire credal set
$\operatorname{Conv}\{Q_1,\ldots,Q_S\}$.
Both arguments are detailed in Appendix~\ref{proofs}.

These results establish the relationship between CDEC and CBDL
with a single specified prior and a finite collection of
likelihood models \citep{ibnn}. Both constructions obtain a
predictive credal set by taking the convex hull of the posterior
predictive distributions associated with the model alternatives.
CDEC obtains these distributions analytically from its learned
evidential posteriors.

Agreement between separately fitted Bayesian and evidential
models can be assessed by comparing their predictive credal
sets. Appendix~\ref{app:stability} gives bounds on the differences
between their uncertainty measures when these are computed
using the same entropy function, used to summarise the uncertainty in each predictive
distribution, as explained in the next section.

\subsection{Aleatoric and Epistemic Uncertainties}\label{sec-3.2}

For a Categorical distribution $P=\operatorname{Cat}(\pi)$ on $k$ labels, its
Shannon entropy in bits is
\begin{equation*}
H(P)=H(\pi)=-\sum_{j=1}^k \pi_j\log_2\pi_j,
\qquad 0\log_20=0.
\end{equation*}
Following \citet{abellan,eyke}, for a predictive credal set $\mathcal P$, define
\[
\underline H(\mathcal P)=\min_{P\in\mathcal P}H(P),\qquad
\overline H(\mathcal P)=\max_{P\in\mathcal P}H(P).
\]
We use the following additive decomposition for the total uncertainty (TU) associated with $\mathcal{P}$,
\begin{equation}\label{decomp}
\underbrace{\overline H(\mathcal P)}_{\mathrm{TU}(\mathcal P)}
=\underbrace{\underline H(\mathcal P)}_{\mathrm{AU}(\mathcal P)}
+\underbrace{\overline H(\mathcal P)-\underline H(\mathcal P)}_{\mathrm{EU}(\mathcal P)}.
\end{equation}
These are uncertainty functionals of the specified predictive model. 
Other decompositions are also possible \citep{eyke,hofman2024proper,chau2025integral}; a large or small
entropy gap must be interpreted together with how the predictive set was formed.

\begin{theorem}[Operative Uncertainty Bounds]\label{thm-imp-reduced}
Let $\mathcal P=\operatorname{Conv}\{P_1,\ldots,P_S\}$ on $k$ labels, with
$S\geq1$; the generators need not all be extreme. Set
\begin{equation*}
A=\min_sH(P_s),\qquad L=\max_sH(P_s),\qquad
u[\mathrm{TU}]=L+\log_2S,\qquad
U=\min\{\log_2k,u[\mathrm{TU}]\}.
\end{equation*}
Then
\begin{equation*}
\mathrm{AU}(\mathcal P)=A,\qquad
L\leq\mathrm{TU}(\mathcal P)\leq U,\qquad
L-A\leq\mathrm{EU}(\mathcal P)\leq U-A.
\end{equation*}
One may apply the same formulas after removing redundant generators.
\end{theorem}
The lower entropy follows from concavity. The upper bound follows by introducing
a latent mixture index and using the entropy chain rule; Appendix~\ref{proofs}
gives the proof. The calculation only requires the entropies of the supplied
generators. We define $u[\mathrm{EU}]=u[\mathrm{TU}]-A$.
The quantities $u[\mathrm{TU}]$ and $u[\mathrm{EU}]$ are
upper bounds on the uncertainty measures in \eqref{decomp}.
Combining these with $\mathrm{TU}\leq\log_2 k$ gives the
bounds $U$ and $U-A$, respectively.

Exact upper entropy can also be expressed as
\begin{equation}\label{eq:cdec-exact}
\mathrm{TU}(\mathcal P)
=\max_{\beta\in\Delta^{S-1}}
H \left(\sum_{s=1}^S\beta_s\pi_s\right),
\end{equation}
where $\pi_s$ is the probability vector of $P_s$ and
$\Delta^{S-1}$ is the $(S-1)$-dimensional simplex.
For positive generators and $f(\beta)=H(\sum_s\beta_s\pi_s)$, concavity gives
\begin{equation}\label{eq:entropy-gap}
f(\beta)\leq\max_w f(w)
\leq f(\beta)+\max_s[\nabla f(\beta)]_s-\beta^\top\nabla f(\beta).
\end{equation}
For any feasible mixture weights $\beta$, the quantity
$\max_s[\nabla f(\beta)]_s-\beta^\top\nabla f(\beta)$
bounds the difference between the maximum entropy
$\mathrm{TU}(\mathcal P)$ and the value $f(\beta)$.
It therefore provides a criterion for stopping the optimisation
once the error bound falls below a prescribed tolerance.

\begin{remark}[Exact Uncertainty for Two Classes]\label{rem:binary}
For $k=2$, each predictive model has a probability vector
$(\pi_{s,1},1-\pi_{s,1})$. The convex hull of these vectors
consists of all $(p,1-p)$ with $p\in[a,b]$, where
\[
a=\min_s\pi_{s,1},
\qquad
b=\max_s\pi_{s,1}.
\]
The Shannon entropy of a two-class distribution is
\[
h_2(p)=-p\log_2p-(1-p)\log_2(1-p),
\]
with $0\log_2 0=0$. This function attains its maximum
over $[a,b]$ at the point closest to $1/2$, namely
$t=\min\{b,\max\{a,1/2\}\}$, and its minimum at an
endpoint. Hence,
\begin{equation}\label{eq:cdec-binary-uncertainties}
\mathrm{TU}=h_2(t),
\qquad
\mathrm{AU}=\min\{h_2(a),h_2(b)\}.
\end{equation}

For $S\geq2$, the general upper bound reduces to
$U=\min\{1,L+\log_2S\}=1$. Consequently, a rule requiring
$U\leq1-\epsilon$, with $\epsilon>0$, abstains on every
input when using this bound. Using the exact TU value in
\eqref{eq:cdec-binary-uncertainties} allows the abstention
rule to be evaluated directly, without numerical optimisation.
\end{remark}

\section{Credal Deep Evidential Classification}\label{ien}

We now use the predictive credal set $\mathcal{P}$ to construct prediction
regions that contain the correct label with at least a prescribed
probability under every distribution in the set. We then describe
how CDEC uses the uncertainty measures introduced in \eqref{decomp} to decide
whether to return a prediction or abstain.

\subsection{Imprecise Highest Density Regions}\label{sec:ihdr}

We first introduce lower and upper probabilities and the corresponding prediction
regions. Let $\mathcal P$ be a nonempty credal set on the finite label space.
Its lower and upper probabilities are
\[
\underline P(A)=\min_{P\in\mathcal P}P(A),\qquad
\overline P(A)=\max_{P\in\mathcal P}P(A),\qquad A\subseteq\mathcal Y,
\]
that is, the lower and upper envelopes of the credal set $\mathcal{P}$. They are conjugate, $\overline P(A)=1-\underline P(A^c)$. For a finitely
generated credal set, the extrema of these linear functionals are attained at
its generators and, in particular, at extreme elements \citep{dantzig,ibnn}.

\begin{definition}[Imprecise Highest Density Region \citep{coolen}]\label{ihdr-def}
For $0<\gamma<1$, a set $IR_\gamma(\mathcal P)\subseteq\mathcal Y$ is a
$(1-\gamma)$-Imprecise Highest Density Region (IHDR) if
$\underline P(IR_\gamma(\mathcal P))\geq1-\gamma$ and its cardinality is minimal
among events satisfying this constraint.
\end{definition}
The notation $IR$ stands for imprecise region. Several prediction regions may satisfy the probability requirement
while containing the same minimum number of labels. Among these,
we choose the one with the largest lower probability, that is,
the largest minimum probability across the distributions in the
credal set. Any remaining ties are resolved by listing each
region's labels in increasing order and selecting the first
list in lexicographic order.

As a simple example, let $\mathcal Y=\{y_1,\ldots,y_5\}$ and let the three generators
be as in Table~\ref{tab_ex}. For $\gamma=0.1$, the smallest event having probability
at least $0.9$ under every generator is $\{y_1,y_2,y_3\}$.
\begin{table}[h!]
\centering
\begin{tabular}{lrrrrr}
\toprule
 & $y_1$ & $y_2$ & $y_3$ & $y_4$ & $y_5$\\
\midrule
$P_1$ & {\color{red}.7} & {\color{red}.25} & .03 & .01 & .01\\
$P_2$ & {\color{red}.6} & {\color{red}.2} & {\color{red}.1} & .05 & .05\\
$P_3$ & {\color{red}.5} & {\color{red}.3} & {\color{red}.15} & .025 & .025\\
\bottomrule
\end{tabular}
\caption{Predictive probabilities for the IHDR example.}
\label{tab_ex}
\end{table}

For $\mathcal P=\operatorname{Conv}\{\operatorname{Cat}(\pi_s):s\leq S\}$,
the exact finite optimisation is
\begin{equation}\label{eq:ihdr-program}
\min_{z\in\{0,1\}^k}\sum_{j=1}^kz_j
\quad\text{subject to}\quad
\sum_{j=1}^k\pi_{s,j}z_j\geq1-\gamma,\qquad s\in\{1,\ldots,S\}.
\end{equation}
The prediction region contains the labels $j$ for which $z_j=1$.
For small label spaces, we can examine the subsets of labels
and select a smallest subset satisfying all probability
constraints, with ties resolved as described above.
Alternatively, a mixed integer optimisation solver can search
over the binary variables $z_j$ to find such a region; every feasible output satisfies the prescribed lower probability
guarantee. Establishing that no smaller feasible region exists
additionally certifies that the output is an exact IHDR.

\begin{remark}[Singleton lower bound]\label{bound-ihdr}
For every event $A$,
\begin{equation*}
\underline P(A)=\min_s\sum_{j\in A}p_{s,j}
\geq\sum_{j\in A}\min_sp_{s,j}.
\end{equation*}
The inequality follows by taking the minimum after summing. It gives a sufficient
mass certificate, but sorting singleton lower probabilities need not solve
\eqref{eq:ihdr-program}. We use the exact event constraint in the experiments.
\end{remark}

An IHDR has mass at least $1-\gamma$ under every member of its predictive set.
If the true conditional distribution belongs to that set, the same statement
holds for it. Membership is an additional modelling assumption that can be tested \citep{pmlr-v258-chau25a}; empirical
coverage measures how the resulting regions behave on held-out observations.

\subsection{Credal Deep Evidential Classification}\label{cdec-sec}

Algorithm~\ref{algo-1} uses an upper certificate to decide when to return a
prediction region. Let $\tau=\log_2k-\epsilon$, with $0<\epsilon<\log_2k$.
If $U\leq\tau$, total uncertainty is at most the threshold. If $L>\tau$,
excessive total uncertainty is certified. For $L\leq\tau<U$, the procedure
abstains because the bounds are inconclusive.

\begin{algorithm}[h!]
\caption{Credal Deep Evidential Classification}\label{algo-1}
\begin{algorithmic}[1]
\Require Training and tuning data; $S\geq1$ evidential model specifications;
$\gamma\in(0,1)$; $\epsilon\in(0,\log_2k)$; method for evaluating TU: analytic bounds, the exact binary
formula (available when $k=2$), or numerical optimisation
\Ensure An IHDR or an abstention diagnostic
\State Fit and select the $S$ evidential models without using final test inputs.
\Statex \textbf{For a new input $\tilde x$:}
\State Obtain $p_s=\pi_s^\prime(\tilde x)$ and
$\mathcal P_{\mathrm{pred}}=\operatorname{Conv}\{\operatorname{Cat}(p_s):s\leq S\}$.
\State Compute $A,L,U$ using Theorem~\ref{thm-imp-reduced}.
\If{the exact binary formula is selected}
    \State Set $L=U$ to the TU value in
    \eqref{eq:cdec-binary-uncertainties}.
\ElsIf{numerical optimisation is selected}
    \State Numerically maximise the entropy in
    \eqref{eq:cdec-exact}.
    \State Update $L$ and $U$ using the lower and upper
    bounds in \eqref{eq:entropy-gap}.
\EndIf
\State Set $\tau=\log_2k-\epsilon$.
\If{$U\leq\tau$}
\State Return an IHDR solving \eqref{eq:ihdr-program}.
\ElsIf{$L>\tau$}
\If{$2A\geq U$}
\State Abstain: excessive TU and AU dominance are certified.
\ElsIf{$2A<L$}
\State Abstain: excessive TU and EU dominance are certified.
\Else
\State Abstain: excessive TU is certified; its dominant component is unresolved.
\EndIf
\Else
\State Abstain: the TU threshold cannot be resolved from the available bounds.
\EndIf
\end{algorithmic}
\end{algorithm}

The dominance certificates follow from $\mathrm{EU}=\mathrm{TU}-A$:
$2A\geq U$ implies $\mathrm{AU}\geq\mathrm{EU}$, whereas $2A<L$ implies
$\mathrm{EU}>\mathrm{AU}$. The condition $U\leq\tau$ certifies that total uncertainty
is at most the prescribed threshold, and $L>\tau$ certifies
that it exceeds that threshold, where
$\tau=\log_2 k-\epsilon$. Exact or optimised methods can change the return/abstain
decision, so results and timings are reported by method.

The parameter $\gamma$ specifies the tolerated probability
of the correct label falling outside the prediction region.
We require this probability to be at most $\gamma$ under
every distribution in the predictive credal set. The abstention margin $\epsilon$ controls proximity to the maximum
possible total uncertainty. 
Our experiments evaluate the IHDR construction on all test
inputs, including those on which Algorithm~\ref{algo-1} would
abstain. We evaluate the complete procedure separately,
reporting the proportion of inputs receiving a prediction
and its performance on those inputs.

\subsection{Features of Credal Deep Evidential Classification}

CDEC evaluates a proposed label or prediction region using
its smallest probability across the fitted evidential models.
For applications requiring a single class, we can retain the
same abstention rule and replace the region returned by
Algorithm~\ref{algo-1} with
\begin{equation}\label{eq:cdec-point}
j^\star\in\arg\max_{j\leq k}\min_{s\leq S}p_{s,j}.
\end{equation}
This selects the class with the largest lower probability,
with ties resolved by a fixed label ordering. Under zero-one
loss, predicting class $j$ has worst case error probability
$1-\min_s \pi_{s,j}$. Equation~\eqref{eq:cdec-point} therefore
minimises this error probability among the $k$ possible class
choices. The selected singleton satisfies the prescribed
$1-\gamma$ probability requirement whenever
$\min_s \pi_{s,j^\star}\geq1-\gamma$.


Other evidential classifiers can supply the virtual counts
used by CDEC, including the model of \citet{pandey23a};
\citet{ulmer} review further evidential constructions.
Their Categorical predictive distributions can be combined
using the same convex hull and IHDR construction.

The lower and upper uncertainty measures can also be defined
using other continuous, concave entropy functions.
For example, the Gini index
$G(p)=1-\sum_{j=1}^k p_j^2$ is the minimum expected Brier
loss \citep{hofman2024proper}. Concavity ensures that the
lower uncertainty is attained at one of the supplied
predictive vectors, while the upper uncertainty is obtained
by maximising the chosen function over their convex hull.
The computational bounds and abstention threshold must
correspond to that function. In particular, the bound
involving $\log_2 S$ in Theorem~\ref{thm-imp-reduced}
applies to Shannon entropy; the maximum Gini value
$1-1/k$ replaces $\log_2 k$ when specifying a Gini
abstention threshold.

CDEC shares with the Imprecise Dirichlet Model and the
Naive Credal Classifier the use of sets of predictive
distributions to represent uncertainty and guide decisions
\citep{walley_marbles,bernard_bom,zaffalon}.
Appendix~\ref{app:idm-ncc} discusses their relationships,
including the sources of predictive imprecision, the
modelling assumptions, and the criteria used to return
one or several labels.
\subsection{CDEC with Corruption-Augmented Evidential Training}
\label{sec:cdec-corruption-training}

This section specifies a training extension of CDEC, and establishes a conditional
stability property of its predictive hull. 
The predictive credal construction can be combined with a training distribution
that includes specified input perturbations. Let $\nu$ be a distribution over
measurable transformations of the input space, fixed using fitting and tuning
information. We train on perturbed inputs using their original labels,
which we assume remain correct after perturbation.

For member $s$, write $\alpha_s(x;\theta_s)$ for the Dirichlet parameters and put
$\alpha_{s,0}(x;\theta_s)=\sum_j\alpha_{s,j}(x;\theta_s)$. The augmented objective is
\begin{equation}\label{eq:cdec-augmented-uce}
\widehat J_s(\theta_s)
=
\frac{1}{N}\sum_{i=1}^N\mathbb E_{T\sim\nu}
\left[
 \psi\bigl(\alpha_{s,0}(Tx_i;\theta_s)\bigr)
 -\psi\bigl(\alpha_{s,y_i}(Tx_i;\theta_s)\bigr)
 -\lambda H \left(\operatorname{Dir}
          (\alpha_s(Tx_i;\theta_s))\right)
\right].
\end{equation}
To train on both original and perturbed inputs, we can
choose $\nu$ so that it sometimes leaves an input unchanged.
Each time a training input is used, we draw a new
transformation from $\nu$ and evaluate the loss on the
resulting input. These sampled losses approximate the
expectation in \eqref{eq:cdec-augmented-uce}.
Each original training observation contributes once to
its class count $n_j$, regardless of how many perturbed
versions are used during training.
We select the fitted networks using separate tuning data; their parameters then remain fixed. 
After fitting, the predictive credal set is
\[
\mathcal P^{\nu}_{\mathrm{pred}}(x)
=\operatorname{Conv}\left\{
 \operatorname{Cat} \left(
 \frac{\alpha_s(x;\widehat\theta_s)}
      {\alpha_{s,0}(x;\widehat\theta_s)}\right):s\in\{1,\ldots,S\}
\right\}.
\]
For each new input, we evaluate each fitted network once
to obtain its class probabilities. We then apply the
uncertainty calculations and abstention rule in
Algorithm~\ref{algo-1}. Whenever the algorithm returns an
IHDR, every distribution in
$\mathcal P^\nu_{\mathrm{pred}}(x)$ assigns that region
probability at least $1-\gamma$.
The fitted CDEC model can also serve as the reference
model for choosing the IDEC interval width through the
calibration procedure in
Section~\ref{sec:radius-calibration}.

The following proposition relates stability of the learned generators to stability
of the credal decision; it is a consequence of convexification and total variation

\begin{proposition}[Generator Stability and Credal Decisions]
\label{prop:cdec-generator-stability}
Let $p_1,\ldots,p_S$ and $q_1,\ldots,q_S$ be probability vectors on the same
finite label space, and let
$\delta=\max_s\operatorname{TV}(p_s,q_s)$, where $\operatorname{TV}$ denotes the total variation distance.
Set $\mathcal P=\operatorname{Conv}\{p_s\}$ and
$\mathcal Q=\operatorname{Conv}\{q_s\}$; then,
\begin{equation}\label{eq:cdec-hull-stability}
d_H^{\operatorname{TV}}(\mathcal P,\mathcal Q)\leq\delta,
\end{equation}
where $d_H^{\operatorname{TV}}$ denotes the Hausdorff metric associated with the total variation distance. For every event $R$,
\[
\left|\underline P_{\mathcal P}(R)
      -\underline P_{\mathcal Q}(R)\right|\leq\delta.
\]
In particular, if $\underline P_{\mathcal P}(R)\geq1-\gamma$, then
$\underline P_{\mathcal Q}(R)\geq1-\gamma-\delta$.
Suppose further that $j^\star$ uniquely maximises
$a_j=\min_s p_{s,j}$. If
\[
a_{j^\star}-\max_{j\ne j^\star}a_j>2\delta,
\]
then $j^\star$ also uniquely maximises $\min_s q_{s,j}$.
\end{proposition}

To apply the proposition to input perturbations, take
the two collections to be the networks' predictions at
$x$ and at its perturbed version $Tx$, respectively.
If each network changes its predictive probabilities
by at most $\delta$ in total variation, an IHDR
constructed at $x$ has probability at least
$1-\gamma-\delta$ under every distribution in the
predictive credal set at $Tx$. The single class rule
in \eqref{eq:cdec-point} also selects the same label
whenever, at $x$, that label's lower probability exceeds
the lower probability of every other label by more
than $2\delta$.

\begin{proof}[Proof of Proposition \ref{prop:cdec-generator-stability}]
For each simplex weight vector $w$, convexity of total variation gives
\[
\operatorname{TV}\left(\sum_s w_sp_s,\sum_s w_sq_s\right)
\leq\sum_s w_s\operatorname{TV}(p_s,q_s)\leq\delta.
\]
Using the same weights in both directions proves the Hausdorff bound \eqref{eq:cdec-hull-stability}.
For each event and member, the corresponding event probabilities differ by at
most $\delta$. Taking minima proves the lower probability inequality and its
coverage consequence. For singleton events, put $b_j=\min_s q_{s,j}$.
The same argument gives $|a_j-b_j|\leq\delta$. Hence, for $j\ne j^\star$,
$b_{j^\star}-b_j\geq a_{j^\star}-a_j-2\delta>0$.
\end{proof}

\section{Interval Deep Evidential Classification}\label{ien-simpl}

IDEC constructs a predictive credal set from one evidential predictor through an
interval of measures \citep{lorraine,coolen}. We first describe the interval and
its uncertainty functionals. We then distinguish a region's robustness radius
from model width selection, and introduce a reference-calibrated deployment mode.

\subsection{Intervals of Measures}\label{int-meas-section}

On the finite label space $\mathcal Y=\{1,\ldots,k\}$,
a nonnegative vector $h$ represents a finite measure
through $\mu_h(A)=\sum_{j\in A}h_j$, for
$A\subseteq\mathcal Y$. We identify each probability
vector $\pi$ with its Categorical distribution
$\operatorname{Cat}(\pi)$.

\begin{definition}[Interval of Measures]\label{int-meas}
Let $\ell=(\ell_1,\ldots,\ell_k)$ and
$u=(u_1,\ldots,u_k)$ be vectors satisfying
\[
0<\ell_j\leq u_j<\infty,
\qquad\text{for every }j\in\mathcal Y.
\]
The interval of measures determined by $\ell$ and $u$
is the set of probability vectors obtained by
normalising the vectors whose coordinates lie
between these bounds,
\[
\mathcal I(\ell,u)
=
\left\{
\frac{h}{\sum_{i=1}^k h_i}:
h\in\mathbb R^k,\ 
\ell_j\leq h_j\leq u_j,
\text{ for every }j\in\mathcal Y
\right\}.
\]
\end{definition}
In IDEC, $\pi\equiv \pi^\prime(\tilde x)$ is the strictly positive evidential predictive
vector, $\ell=\pi$ and $u=(1+d)\pi$, for $d\geq0$. It is convenient to write
\begin{equation*}
\rho=\log(1+d),\qquad
\mathcal K_\rho(\pi)=\mathcal I(\pi,e^\rho\pi).
\end{equation*}
Here $\log$ is the natural logarithm. We define
$\mathcal K_\infty(\pi)=\Delta^{k-1}$ as the closed limiting set; this endpoint is treated separately from finite radius formulas.

\begin{example}\label{ex-int-meas}
For $p=(0.7,0.2,0.08,0.02)^\top$ and $d=1.58$, the unnormalised mass
$h=(1.2,0.3,0.1,0.04)^\top$ satisfies $p\leq h\leq(1+d)p$.
Its normalisation gives the predictive probability vector
$(0.7317,0.1829,0.0610,0.0244)^\top$, to four decimal places.
The admissible $h$ is an unnormalised mass vector, whereas $h/\sum_jh_j$
is an element of the Interval.
\end{example}

We now give a convenient way to compute analytically lower and upper probabilities, in the context of Intervals of Measures.
\begin{proposition}[Lower and Upper Probabilities]\label{prop-cred-int}
The interval $\mathcal I(\ell,u)$ is a compact convex set. For every event
$A\subseteq\mathcal Y$, writing $\ell(A)=\sum_{j\in A}\ell_j$ and similarly
for $u(A)$, its lower and upper probabilities are
\[
\underline P(A)=\frac{\ell(A)}{\ell(A)+u(A^c)},\qquad
\overline P(A)=\frac{u(A)}{u(A)+\ell(A^c)}.
\]
For a set of labels $A\subseteq\mathcal Y$, let
$a=\sum_{j\in A}\pi_j$ denote the probability assigned
to $A$ by $\operatorname{Cat}(\pi)$. Then,
\begin{equation}\label{eq-deriv-2}
\underline P_\rho(A)=\frac{a}{a+e^\rho(1-a)},\qquad
\overline P_\rho(A)=\frac{e^\rho a}{e^\rho a+1-a}.
\end{equation}
\end{proposition}
These Interval of Measures formulas are established results
\citep{lorraine,coolen,walley}. Appendix~\ref{proofs} gives the finite dimensional
argument, including the cases $A=\varnothing$ and $A=\mathcal Y$.

\begin{corollary}[Finite Predictive-Credal Representation]\label{cor-equiv}
For finite $d$, let $B=\prod_{j=1}^k[\pi_j,(1+d)\pi_j]$ and let
$\operatorname{vert}(B)$ be its vertices. Then
\[
\mathcal I(\pi,(1+d)\pi)
=\operatorname{Conv}\left\{\frac{v}{\sum_jv_j}:v\in\operatorname{vert}(B)\right\}.
\]
\end{corollary}
The interval is therefore exactly the convex hull of
finitely many Categorical distributions. IDEC constructs
it directly from one network's class probabilities and
the parameter $d$. To calibrate $d$, we use a fitted
CDEC or CBDL model on separate calibration data.
After calibration, predictions use only the IDEC
network and the selected $d$.

\subsection{Aleatoric and Epistemic Uncertainties in the Interval Setting}
\label{sec-5.2}

For $Y_q\sim\operatorname{Cat}(q)$ let $Z_q=e_{Y_q}$, where $e_j$ is the $j$-th
canonical basis vector. Its total variance is
\begin{equation*}
G(q)=\operatorname{tr}(\operatorname{Cov}(Z_q))
=\sum_jq_j(1-q_j)=1-\sum_jq_j^2.
\end{equation*}
This is the {\em Gini impurity} or {\em quadratic entropy}. It is label-invariant and has the
proper loss interpretation
$\mathbb E_q\|a-e_{Y_q}\|^2=\|a-q\|^2+G(q)$.
Thus, $G(q)$ is the minimum expected quadratic prediction loss
\citep{hofman2024proper}. For the IDEC predictive credal set
$\mathcal K=\mathcal I(\pi,(1+d)\pi)$ described in
Corollary~\ref{cor-equiv}, we define
\begin{align*}
\mathrm{AU}_G(\mathcal K)
=\min_{\tilde\pi\in\mathcal K}G(\tilde\pi),\qquad 
\mathrm{TU}_G(\mathcal K)
=\max_{\tilde\pi\in\mathcal K}G(\tilde\pi),\qquad 
\mathrm{EU}_G(\mathcal K)
=\mathrm{TU}_G(\mathcal K)-\mathrm{AU}_G(\mathcal K).
\end{align*}
The subscript $G$ distinguishes these quantities from Shannon entropy values.
For all $\rho$,
$0\leq\mathrm{AU}_G\leq\mathrm{TU}_G\leq1-1/k$.
At $\rho=0$, AU and TU equal $G(\pi)$ and EU is zero. Set inclusion implies that
AU is nonincreasing and TU and EU are nondecreasing in $\rho$, for fixed $\pi$.

\begin{proposition}[Computation of the IDEC Uncertainty Terms]
\label{prop:gini-computation}
Let $\pi\in\Delta^{k-1}$ satisfy $\pi_j>0$ for
$j=1,\ldots,k$. For $0\leq\rho<\infty$, put
$r=e^{-\rho}$ and arrange the components in decreasing
order, $\pi_{(1)}\geq\cdots\geq\pi_{(k)}>0$.
Define $s_0=a_0=0$ and, for $m=1,\ldots,k$,
\[
s_m=\sum_{i=1}^{m}\pi_{(i)},
\qquad
a_m=\sum_{i=1}^{m}\pi_{(i)}^2.
\]
Then,
\begin{equation}\label{eq:gini-au}
\mathrm{AU}_G(\mathcal K_\rho(\pi))
=
1-\max_{0\leq m\leq k}
\frac{a_m+r^2(a_k-a_m)}
     {[s_m+r(1-s_m)]^2}.
\end{equation}
Moreover,
\begin{equation}\label{eq:gini-tu}
\mathrm{TU}_G(\mathcal K_\rho(\pi))
=
1-
\min_{\substack{
\tilde\pi\in\Delta^{k-1}\\
r\leq\lambda\leq1}}
\left\{
\sum_{j=1}^k\tilde\pi_j^2:
\lambda\pi_j\leq\tilde\pi_j
\leq\frac{\lambda}{r}\pi_j,
\quad j\in\{1,\ldots,k\}
\right\}.
\end{equation}
For the limiting case
$\mathcal K_\infty(\pi)=\Delta^{k-1}$, the corresponding
uncertainty terms are
\[
\mathrm{AU}_G=0,
\qquad
\mathrm{TU}_G=\mathrm{EU}_G=1-\frac{1}{k}.
\]
\end{proposition}
Equation~\eqref{eq:gini-au} computes AU by checking
$k+1$ candidate probability vectors. For each
$m\in\{0,\ldots,k\}$, we multiply the $m$ largest components
of $\pi$ by $e^\rho$, leave the remaining components
unchanged, and normalise the resulting vector.
The smallest Gini value among these candidates is the AU.
The calculation therefore requires one sorting step,
followed by the evaluation of the $k+1$ quantities
in \eqref{eq:gini-au}.

Equation~\eqref{eq:gini-tu} finds the probability vector
in the interval that is closest to the uniform
distribution in Euclidean distance. Its objective is
convex and its constraints are linear, so it is a
convex quadratic programme. Appendix~\ref{proofs}
also derives an equivalent calculation using a
single variable, whose value can be found by
bisection when the uniform distribution is outside
the interval. Once AU and TU have been computed,
EU is their difference.

\subsection{Prediction and Tolerable Perturbation}\label{pred-int}

For a prediction region $R\subseteq\mathcal Y$ and a
chosen finite width $d=e^\rho-1$, \eqref{eq-deriv-2} gives
\begin{equation}\label{eq-obj}
\underline P_\rho(R)\geq1-\gamma
\quad\Longleftrightarrow\quad
\sum_{j\in R}\pi_j
\geq1-\frac{\gamma}{1+(1-\gamma)d},
\qquad 0<\gamma<1.
\end{equation}
Hence an IHDR consists of the most probable labels up to the first cumulative
sum reaching the threshold \citep{coolen}. A fixed ordering resolves ties.
For fixed $\pi$ and $\gamma$, increasing $d$ raises the
probability threshold in \eqref{eq-obj}. The IHDR may
therefore need to include more labels to satisfy the
same lower probability requirement. If $\rho=\infty$, the only valid region
is $\mathcal Y$, whose lower probability is one.

\begin{proposition}[Robustness Radius of a Fixed Region]
\label{prop:robust-radius}
Let $\pi=(\pi_1,\ldots,\pi_k)$ be a probability vector
with $\pi_j>0$ for every $j\in\mathcal Y$. Fix
$0<\gamma<1$ and a prediction region
$R\subseteq\mathcal Y$, and write 
$\pi(R)=\sum_{j\in R}\pi_j$. 
For a finite perturbation parameter $d\geq0$, let
$\rho=\log(1+d)$.
If $1-\gamma\leq\pi(R)<1$, the requirement
$\underline P_\rho(R)\geq1-\gamma$ holds exactly when
\[
0\leq d\leq d_{\mathrm{rob}}(R,\pi,\gamma),
\]
where
\begin{equation*}
d_{\mathrm{rob}}(R,\pi,\gamma)
=
\frac{\pi(R)-(1-\gamma)}
     {(1-\gamma)[1-\pi(R)]}.
\end{equation*}
At $d=d_{\mathrm{rob}}(R,\pi,\gamma)$, the lower
probability of $R$ equals $1-\gamma$.

If $\pi(R)=1$, then $R=\mathcal Y$ and its lower
probability is one for every width, including the
limiting case $\rho=\infty$. We therefore set
$d_{\mathrm{rob}}(R,\pi,\gamma)=\infty$.
If $\pi(R)<1-\gamma$, no nonnegative value of $d$
satisfies the requirement for this fixed region.
\end{proposition}

Proposition \ref{prop:robust-radius} follows from \eqref{eq-deriv-2}, which gives
\[
\underline P_\rho(R)
=
\frac{\pi(R)}{1+d[1-\pi(R)]}.
\]
For $\pi(R)<1$, solving
$\underline P_\rho(R)\geq1-\gamma$ for $d$ gives
the stated bound.

Let $R_\gamma^\star$ be the reference region obtained
by adding labels in decreasing order of $\pi_j$
until their cumulative probability is strictly
greater than $1-\gamma$; a fixed label ordering
resolves ties. Define
\[
d^\star
=
d_{\mathrm{rob}}(R_\gamma^\star,\pi,\gamma).
\]
Thus, $d^\star$ is the largest perturbation for which
every distribution in the IDEC predictive credal set
still assigns probability at least $1-\gamma$ to
this fixed region.

In Example~\ref{ex-int-meas},
$R_{0.05}^\star=\{1,2,3\}$ has probability $0.98$
under $\pi$, giving
$d^\star
=
\frac{0.98-0.95}{0.95(1-0.98)}
\simeq1.58$.
At this value, the lower probability of
$\{1,2,3\}$ is exactly $0.95$.

The parameter $d$ determines the predictive credal
set used by IDEC. Once $d$ has been selected,
we construct the IHDR using \eqref{eq-obj}.
A finite choice $d>d^\star$ requires additional
labels beyond $R_\gamma^\star$ to satisfy the
same lower probability requirement.
The radius $d^\star$ therefore describes the
perturbation tolerated by the fixed region
$R_\gamma^\star$; selecting $d$ determines the
predictive imprecision used to construct a region.


\subsection{Criteria for Choosing the Interval Width}
\label{sec:radius-geometry}

Choosing $d$ requires a criterion for what the resulting
predictive credal set should represent. The predictive class probabilities $\pi_1,\ldots,\pi_k$
specify the mean of the evidential distribution, but
do not determine how concentrated that distribution
is around its mean. For example, for a strictly
positive probability vector $\pi$, the distribution
$\operatorname{Dir}(\kappa\pi)$ has mean $\pi$ for every
$\kappa>0$, while its dispersion decreases as $\kappa$
increases. Moreover, probability scores such as the
Brier score and logarithmic loss, evaluated at the
same vector $\pi$, remain unchanged as $d$ varies.

One possible criterion is to choose the smallest interval
containing a specified collection of predictive
distributions. Another is to require the interval to
contain at least a proportion $1-\eta$ of the probability
mass of a fitted evidential distribution, for a chosen
$0<\eta<1$. In the second case, the probability is
calculated under that fitted distribution, whose
interpretation as a representation of epistemic
uncertainty depends on the evidential modelling and
training assumptions \citep{mira}.
We examine both criteria below.

For two strictly positive probability vectors
$\pi,\tilde\pi\in\Delta^{k-1}$, Hilbert's projective
metric \citep{cohen2024hilbert} is
\begin{equation}\label{eq:hilbert}
    H_{\mathrm{proj}}(\pi,\tilde\pi)
=
\max_{1\leq j\leq k}
\log\frac{\tilde\pi_j}{\pi_j}
-
\min_{1\leq j\leq k}
\log\frac{\tilde\pi_j}{\pi_j}.
\end{equation}
For every finite $\rho\geq0$, the IDEC interval satisfies
\begin{equation*}
\tilde\pi\in\mathcal K_\rho(\pi)
\quad\Longleftrightarrow\quad
H_{\mathrm{proj}}(\pi,\tilde\pi)\leq\rho;
\end{equation*}
here $\mathcal K_\rho(\pi)$ is the IDEC predictive
credal set $\mathcal K$, with its dependence on the
predictive vector $\pi$ and the interval width
$\rho=\log(1+d)$ displayed explicitly. When these
quantities are fixed, we use the shorter notation
$\mathcal K$. 
The ratios $\tilde\pi_j/\pi_j$ describe proportional
changes in the class probabilities. Consequently, a
small absolute change in the probability of an unlikely
class can require a large interval width.

\begin{theorem}[Limitations of Multiplicative Containment]
\label{thm:geometry}
\textnormal{(i) Two nearby predictive distributions.}
For $0<t\leq1/2$, let
\[
\pi_t=
\left(\frac{1-t}{2},\frac{1-t}{2},t\right),
\qquad
\tilde\pi_t=
\left(\frac{1-t^2}{2},\frac{1-t^2}{2},t^2\right).
\]
Their total variation distance satisfies
\[
\operatorname{TV}(\pi_t,\tilde\pi_t)
=t-t^2\longrightarrow0
\qquad\text{as }t\downarrow0.
\]
The smallest radius for which
$\tilde\pi_t\in\mathcal K_{\rho_t}(\pi_t)$ is
\[
\rho_t=\log\frac{1+t}{t}
\longrightarrow\infty.
\]
The corresponding Gini uncertainty terms satisfy
\[
\mathrm{EU}_G\bigl(\mathcal K_{\rho_t}(\pi_t)\bigr)
\longrightarrow\frac23,
\qquad
\mathrm{EU}_G\bigl(
\operatorname{Conv}\{\pi_t,\tilde\pi_t\}
\bigr)
\longrightarrow0.
\]

\medskip
\textnormal{(ii) A concentrating Dirichlet distribution.}
For $n\in\mathbb N$, let
$\Pi_n=(\Pi_{n,1},\Pi_{n,2},\Pi_{n,3})$
be a random probability vector with distribution
$\operatorname{Dir}(n,n,1)$. Its mean, which gives
the Categorical predictive probabilities, is
\[
\pi_n
=\mathbb E[\Pi_n]
=
\left(
\frac{n}{2n+1},
\frac{n}{2n+1},
\frac{1}{2n+1}
\right).
\]
Fix $0<\eta<1$ and define the smallest radius
containing probability mass at least $1-\eta$ by
\[
\rho_{\eta,n}
=
\inf\left\{
\rho\geq0:
\Pr\bigl(\Pi_n\in\mathcal K_\rho(\pi_n)\bigr)
\geq1-\eta
\right\}.
\]
As $n\to\infty$,
\[
\Pi_n\xrightarrow{\Pr}
\left(\frac12,\frac12,0\right),
\qquad
\sum_{j=1}^3\operatorname{Var}(\Pi_{n,j})
\longrightarrow0.
\]
The selected radii satisfy
\[
\rho_{\eta,n}\longrightarrow\rho_\eta>0,
\]
where $\rho_\eta$ is the unique positive solution of
\[
e^{-e^{-\rho_\eta}}-e^{-e^{\rho_\eta}}
=1-\eta.
\]
Furthermore,
\[
\mathrm{EU}_G\bigl(
\mathcal K_{\rho_{\eta,n}}(\pi_n)
\bigr)
\longrightarrow
\frac{(e^{\rho_\eta}-1)^2}
     {2(1+e^{\rho_\eta})^2}
>0.
\]
\end{theorem}

The proof is given in Appendix~\ref{proofs}.
Part (i) shows that two predictive distributions can
become arbitrarily close while their smallest enclosing
IDEC interval retains a large EU score. Their convex
hull has an EU score tending to zero, so the difference
arises from the additional distributions admitted by
the interval. The third class drives this behaviour:
changing its probability from $t$ to $t^2$ is a small
absolute change but a large proportional change.
Accommodating that change with a common multiplicative
width also allows substantial changes in the other
class probabilities.

Part (ii) concerns a different requirement: containing
a fixed probability mass under a Dirichlet distribution.
The Dirichlet distribution becomes concentrated around
$(1/2,1/2,0)$, while the smallest radius satisfying
this requirement converges to $\rho_\eta>0$.
To see what this permits, multiply $\pi_{n,1}$ by
$e^{\rho_{\eta,n}}$, leave the other two components
unchanged, and normalise. The resulting vector belongs
to the IDEC interval and converges to
\[
\left(
\frac{e^{\rho_\eta}}{1+e^{\rho_\eta}},
\frac{1}{1+e^{\rho_\eta}},
0
\right).
\]
Its first two probabilities remain unequal because
$\rho_\eta>0$. Thus, the interval continues to admit
alternatives that remain separated from the vector
around which the Dirichlet distribution concentrates.
Any larger width satisfying the same probability
mass requirement also retains a positive limiting
lower bound on its EU score, since the intervals
are nested and $\mathrm{EU}_G$ is nondecreasing
under set inclusion.
Figure~\ref{fig:boundary} illustrates this
noncontraction in the $90\%$ credible case.

\begin{figure}[t]
\centering
\includegraphics[width=.7\textwidth]{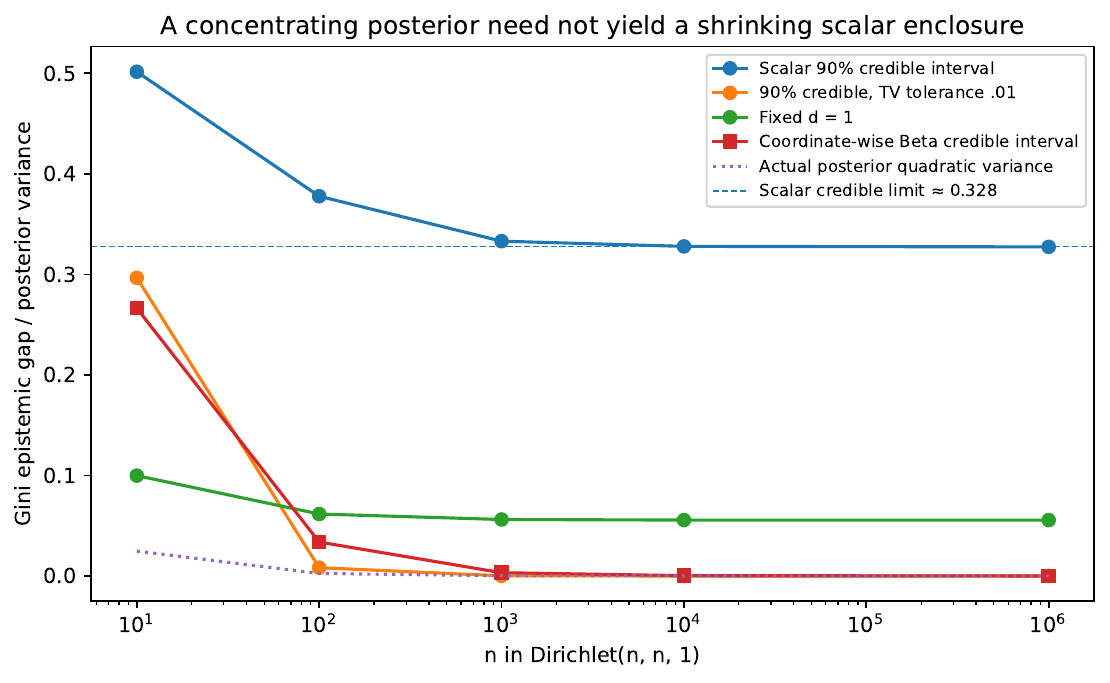}
\caption{Synthetic illustration for $\operatorname{Dir}(n,n,1)$.
The scalar $90\%$ credible enclosure has limiting Gini gap about $0.328$,
while posterior quadratic variance tends to zero. The coordinate-wise credible
box and absolute discrepancy tolerant alternative contract in this example.
A fixed positive scalar radius retains width on the limiting two-class face.
The theoretical limit is established in Theorem~\ref{thm:geometry}.}
\label{fig:boundary}
\end{figure}

These counterexamples show that containment alone
does not ensure that the interval's EU score decreases
as the distributions being represented become close
or concentrated. They motivate a criterion that
allows small discrepancies in absolute class
probabilities.

In Section~\ref{sec:radius-calibration}, a fitted CDEC
or CBDL model supplies the reference predictive credal
set. For each reference probability vector, we measure
its total variation distance to the closest vector
in the IDEC set. The largest of these distances is
the approximation error at that input.
Calibration selects $d$ to control this error on
average. For losses in $[0,1]$, this also bounds,
on average, how much IDEC can underestimate the
reference model's worst case expected loss when
evaluating a decision.
\subsection{Selecting the Interval Width Using a Reference Model}
\label{sec:radius-calibration}

We choose the IDEC interval width by comparing its associated predictive
credal sets with those of a fitted reference, which
may be a CDEC or CBDL model. Our objective is to control how much
IDEC can underestimate the reference model's assessment of
decision risk, while using a single evidential predictor for
subsequent predictions.

Let $\pi(x)\in\Delta^{k-1}$ be the Categorical predictive vector
of the single evidential network, with $\pi_j(x)>0$ for every
class $j$. Let
\[
\mathcal P_{\mathrm{ref}}(x)
=
\operatorname{Conv}
\{\pi_1^{\mathrm{ref}}(x),\ldots,\pi_J^{\mathrm{ref}}(x)\}
\]
be the reference predictive credal set, where $J$ is the number
of reference predictive vectors retained at each input.
For a candidate radius $\rho$, IDEC uses
$\mathcal K_\rho(\pi(x))$, with $d=e^\rho-1$.

Both models are fitted and selected using data separate from
the calibration inputs. Before calibration, we also fix all
preprocessing, tuning choices and numerical settings that
determine their predictions. If the Bayesian reference uses
a finite collection of posterior weight samples, that collection
is fixed and reused when defining its predictive vectors.
The calibration guarantee then concerns this specified
approximation to the Bayesian reference.

For probability vectors $\pi^{\mathrm{ref}}$ and
$\tilde{\pi}$, write
$\operatorname{TV}(\pi^{\mathrm{ref}},\tilde{\pi})
=
\frac12\sum_{j=1}^k
\left|\pi_j^{\mathrm{ref}}-\tilde{\pi}_j\right|$.
We first measure the distance from a single reference vector
to the closest distribution in the IDEC set,
\begin{equation*}
\delta_\rho(\pi^{\mathrm{ref}},\pi)
=
\min_{\tilde{\pi}\in\mathcal K_\rho(\pi)}
\operatorname{TV}(\pi^{\mathrm{ref}},\tilde{\pi}).
\end{equation*}
We measure the approximation error over the entire reference
credal set by
\begin{equation}\label{eq:reference-distance}
D_\rho(x)
=
\max_{\pi^{\mathrm{ref}}\in\mathcal P_{\mathrm{ref}}(x)}
\delta_\rho(\pi^{\mathrm{ref}},\pi(x))
=
\max_{1\leq s\leq J}
\delta_\rho(\pi_s^{\mathrm{ref}}(x),\pi(x)).
\end{equation}
Since $\mathcal K_\rho(\pi(x))$ is convex,
$\delta_\rho(\pi^{\mathrm{ref}},\pi(x))$ is convex in
$\pi^{\mathrm{ref}}$, giving the second equality in
\eqref{eq:reference-distance}.
The error $D_\rho(x)$ lies in $[0,1]$, is continuous
for finite $\rho$, and is nonincreasing because
$\mathcal K_\rho(\pi(x))$ expands as $\rho$ increases.
It tends to zero as $\rho\to\infty$, with
$D_\infty(x)=0$ at the full-simplex limit
$\mathcal K_\infty(\pi(x))=\Delta^{k-1}$.

\paragraph{Meaning for Decisions.}
Consider a decision whose loss is $f_j$ when the true
class is $j$, where $f=(f_1,\ldots,f_k)\in[0,1]^k$.
For a nonempty compact predictive credal set $\mathcal A$,
define
\[
h_{\mathcal A}(f)
=
\max_{\tilde{\pi}\in\mathcal A}
\sum_{j=1}^k\tilde{\pi}_j f_j.
\]
This is the largest expected loss assigned to the decision
by any distribution in $\mathcal A$.

\begin{proposition}[Decision Interpretation of the Approximation Error]
\label{prop:risk-duality}
For every input $x$ and every $\rho\in[0,\infty]$,
\begin{equation*}
D_\rho(x)
=
\max_{f\in[0,1]^k}
\left\{
h_{\mathcal P_{\mathrm{ref}}(x)}(f)
-
h_{\mathcal K_\rho(\pi(x))}(f)
\right\}.
\end{equation*}
\end{proposition}

Thus, $D_\rho(x)$ is exactly the largest amount by which
IDEC can underestimate the reference model's worst case
expected loss, considering all loss vectors in $[0,1]^k$.
The identity follows from convex minimax duality;
Appendix~\ref{proofs} gives the proof.

The error $D_\rho(x)$ is zero whenever the IDEC set
$\mathcal K_\rho(\pi(x))$ contains the entire reference
set $\mathcal P_{\mathrm{ref}}(x)$, even if the two sets
differ substantially. Distributions included by IDEC
but absent from the reference can increase the worst
case expected loss of a decision. They can also require
larger prediction regions, since the prescribed probability
requirement must hold under every included distribution.
We therefore evaluate prediction region size and uncertainty
ranking alongside the approximation error $D_\rho(x)$.

\paragraph{Computing the Approximation Error.}
The following formulas evaluate the distance from one
reference vector to the IDEC set. They also determine
the smallest radius needed to attain a specified
distance tolerance.

\begin{proposition}[Explicit Distance and Inverse Radius]
\label{prop:tv-formulas}
Let $\pi,\pi^{\mathrm{ref}}\in\Delta^{k-1}$, with
$\pi_j>0$ for every $j$.
Choose an ordering $j_1,\ldots,j_k$ of the class labels
such that
$\frac{\pi_{j_1}^{\mathrm{ref}}}{\pi_{j_1}}
\leq\cdots\leq
\frac{\pi_{j_k}^{\mathrm{ref}}}{\pi_{j_k}}$. 
Write $P_i,Q_i$ for the masses of the first $i$ labels,
and $P_l^{\mathrm{tail}},Q_l^{\mathrm{tail}}$ for the
masses of the last $l$ labels in this ordering:
\[
\begin{aligned}
P_i&=\sum_{a=1}^{i}\pi_{j_a},
&
Q_i&=\sum_{a=1}^{i}\pi_{j_a}^{\mathrm{ref}},
\\
P_l^{\mathrm{tail}}
&=\sum_{a=k-l+1}^{k}\pi_{j_a},
&
Q_l^{\mathrm{tail}}
&=\sum_{a=k-l+1}^{k}\pi_{j_a}^{\mathrm{ref}},
\end{aligned}
\qquad i,l\in\{1,\ldots,k\}.
\]
For a finite radius $\rho\geq0$, put $r=e^{-\rho}$.
Then,
\begin{equation}\label{eq:tv-distance}
\delta_\rho(\pi^{\mathrm{ref}},\pi)
=
\max_{1\leq i,l\leq k}
\left[
\frac{
rP_iQ_l^{\mathrm{tail}}-P_l^{\mathrm{tail}}Q_i
}{
rP_i+P_l^{\mathrm{tail}}
}
\right]_+,
\end{equation}
where $[z]_+=\max\{z,0\}$.
For a distance tolerance $0\leq\xi<1$, define
\[
t_\xi
=
\min\left\{
1,\min_{1\leq i\leq k}\frac{\xi+Q_i}{P_i}
\right\},
\qquad
u_\xi
=
\max\left\{
1,\max_{1\leq l\leq k}
\frac{Q_l^{\mathrm{tail}}-\xi}{P_l^{\mathrm{tail}}}
\right\}.
\]
The smallest radius achieving distance at most $\xi$ is
\begin{equation}\label{eq:inverse-radius}
s_\xi(\pi^{\mathrm{ref}};\pi)
=
\inf\left\{
\rho\geq0:
\delta_\rho(\pi^{\mathrm{ref}},\pi)\leq\xi
\right\}
=
\log\frac{u_\xi}{t_\xi},
\end{equation}
provided $t_\xi>0$.
If $t_\xi=0$, no finite radius attains the tolerance,
and we set $s_\xi(\pi^{\mathrm{ref}};\pi)=\infty$.
The limiting set $\mathcal K_\infty(\pi)=\Delta^{k-1}$
has distance zero from every reference vector.
\end{proposition}

For a chosen radius $\rho$, \eqref{eq:tv-distance}
gives the smallest total variation distance between
the reference prediction $\pi^{\mathrm{ref}}$ and
a distribution in the IDEC set $\mathcal K_\rho(\pi)$.
If we instead specify a tolerated error $\xi$,
\eqref{eq:inverse-radius} gives the smallest radius
needed to keep this distance at or below $\xi$.
The masses $P_i$ and $P_l^{\mathrm{tail}}$ are strictly
positive, so all denominators in the finite radius
formula are positive; the reference vector may have
zero components.

For one pair of probability vectors, sorting takes
$O(k\log k)$ operations. After sorting, the inverse
radius requires $O(k)$ operations, whereas direct
evaluation of the double maximum in
\eqref{eq:tv-distance} requires $O(k^2)$.
The ordering and cumulative sums can be reused when
evaluating different radii.

For the entire reference predictive credal set
$\mathcal P_{\mathrm{ref}}(x)$, the smallest radius
achieving $D_\rho(x)\leq\xi$ is
\[
\max_{1\leq s\leq J}
s_\xi(\pi_s^{\mathrm{ref}}(x);\pi(x)).
\]
The calibration procedure below evaluates
$D_\rho(X_i)$ at a common candidate radius and
averages these errors across calibration inputs.
Appendix~\ref{proofs} proves the formulas and explains
why evaluating the $J$ reference vectors suffices
to compute $D_\rho(x)$.

\paragraph{Selecting a Common Radius by Calibration.}
Let $X_1,\ldots,X_m$ be calibration inputs, with $m\geq1$,
and let $X_{m+1}$ be a new input.
The calibration inputs are separate from the data used
for fitting, tuning, architecture selection and selection
of the reference model. Their labels are not required,
because $D_\rho(X_i)$ depends only on the two models'
predictions.

Let $\mathcal F_{\mathrm{fit}}$ denote all information
fixed before calibration, including the fitted models,
their tuning choices and any random samples used to
approximate the reference predictions.
We condition on this information when stating the
statistical guarantees.

Before inspecting the calibration inputs, choose an
error budget $\zeta\in[0,1)$.
This budget specifies the tolerated expected reference
approximation error. It is separate from the region
error level $\gamma$ and the abstention margin $\epsilon$.
A smaller budget requires closer reference approximation
and can lead to larger IDEC sets and prediction regions.

For $1/(m+1)<\zeta<1$, select
\begin{equation}\label{eq:radius-crc}
\hat\rho_\zeta
=
\inf\left\{
\rho\geq0:
\frac{1+\sum_{i=1}^{m}D_\rho(X_i)}{m+1}
\leq\zeta
\right\},
\qquad
\hat d_\zeta=e^{\hat\rho_\zeta}-1.
\end{equation}
The additional $1$ accounts for the largest possible
error at an unseen input. This is the conformal risk
control correction for a loss bounded by one
\citep{angelopoulos2024crc}.

A finite feasible radius exists because the corrected
empirical error tends to $1/(m+1)<\zeta$ as
$\rho\to\infty$; continuity implies that the infimum
is attained.
For $0\leq\zeta\leq1/(m+1)$, we use
$\hat\rho_\zeta=\infty$, giving the full simplex
and zero reference approximation error.

\begin{theorem}[Expected Reference Approximation Control]
\label{thm:radius-crc}
Suppose that, conditional on $\mathcal F_{\mathrm{fit}}$,
the inputs $X_1,\ldots,X_{m+1}$ are exchangeable.
The predictors, the interval family and the budget
$\zeta$ are fixed before observing the calibration inputs.
Then
\begin{equation}\label{eq:reference-guarantee}
\mathbb E\left[
D_{\hat\rho_\zeta}(X_{m+1})
\mid\mathcal F_{\mathrm{fit}}
\right]
\leq\zeta.
\end{equation}
For $\zeta>1/(m+1)$, the selected radius is the
smallest radius satisfying the corrected empirical
criterion in \eqref{eq:radius-crc}.
\end{theorem}
Theorem \ref{thm:radius-crc} guarantees that the selected width keeps
IDEC's expected approximation error relative to the
reference model at most $\zeta$.
By Proposition~\ref{prop:risk-duality}, the average
amount by which IDEC underestimates the reference
model's worst case expected loss is therefore also
at most $\zeta$, for any loss taking values in $[0,1]$.

The expectation averages over both the calibration
inputs, which determine $\hat\rho_\zeta$, and the
new input on which it is evaluated. Exchangeability
means that their joint distribution, conditional on
$\mathcal F_{\mathrm{fit}}$, is unchanged by permuting
their order; independent inputs from the same population
satisfy this condition.

Theorem \ref{thm:radius-crc} applies conformal risk control to the
specific approximation error in
\eqref{eq:reference-distance}.
Its minimality statement concerns the corrected
empirical criterion. The optional procedure described
below provides a separate guarantee for the population
error of a predictor, after its calibration sample has
been observed. The guarantee also holds for any radius at least
$\hat\rho_\zeta$, because $D_\rho(x)$ is
nonincreasing in $\rho$.


\paragraph{Consequences for Decisions and Prediction Regions.}
For a set of labels $R\subseteq\mathcal Y$, write
\[
\underline P_{\mathrm{ref},x}(R)=
\min_{\pi^{\mathrm{ref}}\in\mathcal P_{\mathrm{ref}}(x)}
\sum_{j\in R}\pi_j^{\mathrm{ref}},
\qquad 
\underline P_{\rho,x}(R)=
\min_{\tilde{\pi}\in\mathcal K_\rho(\pi(x))}
\sum_{j\in R}\tilde{\pi}_j.
\]
These are the lower probabilities of the event $Y\in R$
induced by the reference and IDEC predictive credal sets,
respectively, at input $x$. Let $\hat R(x)$ be an IHDR constructed from the
calibrated IDEC set
$\mathcal K_{\hat\rho_\zeta}(\pi(x))$ at input $x$, 
$\hat R(x)\in
\argmin_{
R\subseteq\mathcal Y \text{, }
\underline P_{\hat\rho_\zeta,x}(R)\geq1-\gamma
}
|R|.$
For a finite radius, use the construction in
\eqref{eq-obj}, ordering labels by decreasing
$\pi_j(x)$ and resolving ties by a fixed label ordering.
If $\hat\rho_\zeta=\infty$, set
$\hat R(x)=\mathcal Y$; the hat indicates dependence on the calibration inputs
through $\hat\rho_\zeta$.
This candidate region is defined before applying
the abstention rule.

\begin{corollary}[Decision and Region Transfer]
\label{cor:reference-transfer}
For every input $x$, decision $a$, loss satisfying
$\ell(a,j)\in[0,1]$ for all labels $j$, and region
$R\subseteq\mathcal Y$,
\begin{align*}
\max_{\pi^{\mathrm{ref}}\in\mathcal P_{\mathrm{ref}}(x)}
\mathbb E_{\pi^{\mathrm{ref}}}\ell(a,Y)
&\leq
\max_{\tilde{\pi}\in\mathcal K_\rho(\pi(x))}
\mathbb{E}_{\tilde{\pi}}\ell(a,Y)
+D_\rho(x),
\\
\underline P_{\mathrm{ref},x}(R)
&\geq
\underline P_{\rho,x}(R)-D_\rho(x).
\end{align*}

Suppose now that, for every input and every realised
calibration sample, 
$\underline P_{\hat\rho_\zeta,x}
\bigl(\hat R(x)\bigr)\geq1-\gamma$. 
Under the assumptions of
Theorem~\ref{thm:radius-crc},
\begin{equation}\label{eq:reference-region-coverage}
\mathbb E\left[
\underline P_{\mathrm{ref},X_{m+1}}
\bigl(\hat R(X_{m+1})\bigr)
\mid\mathcal F_{\mathrm{fit}}
\right]
\geq1-\gamma-\zeta.
\end{equation}

For a true label coverage statement, suppose that the
conditional class probabilities of the new label satisfy
\[
\Pr \left(
Y_{m+1}=j
\mid X_1,\ldots,X_m,X_{m+1},\mathcal F_{\mathrm{fit}}
\right)
=
\pi_j^\star(X_{m+1}),
\qquad j\in\{1,\ldots,k\}.
\]
Thus, conditional on the new input and the fixed fitting
information, observing the calibration inputs does not
change its class probabilities.
Define the reference approximation error for this
conditional distribution by
\[
b_{\mathrm{ref}}(x)
=
\min_{\pi^{\mathrm{ref}}\in\mathcal P_{\mathrm{ref}}(x)}
\operatorname{TV}
\bigl(\pi^\star(x),\pi^{\mathrm{ref}}\bigr).
\]
Then,
\begin{equation}\label{eq:true-label-coverage}
\Pr \left(
Y_{m+1}\in\hat R(X_{m+1})
\mid\mathcal F_{\mathrm{fit}}
\right)\geq
1-\gamma-\zeta
-
\mathbb E\left[
b_{\mathrm{ref}}(X_{m+1})
\mid\mathcal F_{\mathrm{fit}}
\right].
\end{equation}
\end{corollary}

Equation~\eqref{eq:reference-region-coverage}
guarantees that the reference lower probability of
$\hat R(X_{m+1})$ is at least $1-\gamma-\zeta$
on average over the calibration inputs and the new input.
Equation~\eqref{eq:true-label-coverage} bounds the
probability that this region contains the true label,
under the additional assumption on the conditional
label law stated in the corollary.
The extra term
$\mathbb E[b_{\mathrm{ref}}(X_{m+1})
\mid\mathcal F_{\mathrm{fit}}]$
is the expected total variation distance from the true
conditional class distribution to the closest member
of the reference predictive credal set.
If $\pi^\star(x)\in\mathcal P_{\mathrm{ref}}(x)$
at every input, this term is zero and the true marginal
coverage is at least $1-\gamma-\zeta$.
These statements concern candidate regions constructed
for every input, before applying the abstention rule.
Coverage among accepted inputs is a separate conditional
quantity.

The reference is evaluated during calibration.
After calibration, new predictions use the single
evidential network and the selected width; reference fitting and evaluation remain part of the
total computational cost of the calibrated procedure.
A width specified directly by the user instead defines
a sensitivity analysis through the chosen IDEC set.

\paragraph{Calibration with a High Probability Guarantee.}
An alternative procedure uses $m\geq2$ independent
calibration inputs from the target population and
a finite set $\mathcal G\subset[0,\infty)$ of candidate
values for the radius $\rho$, chosen before calibration.
For each candidate, we compute the average calibration
error and add a margin for sampling uncertainty,
obtaining an upper confidence bound $U_\rho$; an empirical Bernstein inequality determines these
margins \citep{maurer2009empirical}.

For a prescribed failure probability $\delta\in(0,1)$,
the margins ensure that, with probability at least
$1-\delta$ over the calibration sample, every candidate
radius has expected approximation error---which refers to a new independent input
from the same population---at most
its corresponding bound $U_\rho$.

We select the smallest radius in $\mathcal G$ whose
bound satisfies $U_\rho\leq\zeta$.
If no candidate qualifies, we use the full simplex,
whose approximation error is trivially zero.
The selected predictor therefore has expected
approximation error at most $\zeta$ with probability
at least $1-\delta$ over the calibration sample.
Appendix~\ref{app:hp-calibration} gives the formulas
and proves this guarantee.
This differs from Theorem~\ref{thm:radius-crc},
whose expectation also averages over calibration
samples. Both procedures use the reference and
prediction functions fixed before calibration,
and their sampling assumptions refer to the
population on which subsequent predictions are
evaluated.

\paragraph{Calibration for Specified Corruption Distributions.}
Let $Q$ be the distribution of a transformed input $TX$,
where $T$ is drawn independently of $X$ from a fixed
transformation distribution $\nu$.
To calibrate for this population, we use transformed
calibration inputs generated according to $Q$.
Theorem~\ref{thm:radius-crc} then applies when such transformed calibration
inputs and the future transformed input are
exchangeable, conditional on
$\mathcal F_{\mathrm{fit}}$.

Training on corrupted examples does not automatically
extend a guarantee obtained from clean calibration inputs
to corrupted inputs. We obtain this guarantee by using calibration inputs
that are jointly exchangeable with future corrupted
inputs, as required by Theorem~\ref{thm:radius-crc}.

For several specified distributions of corrupted inputs,
we calibrate one radius for each distribution and then
use the largest of these radii. Since $D_\rho(x)$ is
nonincreasing in $\rho$, this common radius preserves
the guarantee obtained for each distribution.

\begin{proposition}[A Common Radius for Several Input Populations]
\label{prop:corruption-family-calibration}
Treat the fitted predictors as fixed.
Let $Q_1,\ldots,Q_B$ be specified input distributions.
For each $b\in\{1,\ldots,B\}$, suppose a calibration
procedure produces a radius $\hat\rho_b$
satisfying
\[
\mathbb E\left[
D_{\hat\rho_b}
\bigl(X_{\mathrm{new}}^{(b)}\bigr)
\right]
\leq\zeta_b,
\qquad
X_{\mathrm{new}}^{(b)}\sim Q_b.
\]
Each calibration sample and its corresponding new
input must satisfy the sampling assumptions of
that procedure; calibration samples for different
populations may be dependent.
Define the common radius
\[
\hat\rho
=
\max_{1\leq b\leq B}\hat\rho_b.
\]
Then, for every $b\in\{1,\ldots,B\}$, we have that
\[
\mathbb E\left[
D_{\hat\rho}
\bigl(X_{\mathrm{new}}^{(b)}\bigr)
\right]
\leq\zeta_b.
\]

Alternatively, suppose each procedure guarantees
population approximation error at most $\zeta_b$
with failure probability at most $\delta_b$.
The common radius satisfies all these population
bounds simultaneously with probability at least
$1-\sum_{b=1}^B\delta_b$.
\end{proposition}

The proposition allows IDEC to use one radius
$\hat\rho$ for all specified input distributions
$Q_1,\ldots,Q_B$, while retaining the reference
approximation guarantee obtained separately for each
distribution.
Using this common radius may produce larger prediction
regions than using the radius calibrated for a particular
distribution.

\begin{proof}[Proof of Proposition \ref{prop:corruption-family-calibration}]
For every input $x$ and each $b$,
$\hat\rho\geq\hat\rho_b$ implies
\[
D_{\hat\rho}(x)
\leq D_{\hat\rho_b}(x),
\]
because the approximation error is nonincreasing
in the radius. Taking expectations proves the
first conclusion.

For the high probability conclusion, the same
inequality preserves each population bound
whenever its individual guarantee holds.
The union bound gives probability at least
$1-\sum_{b=1}^B\delta_b$ that all the individual
guarantees hold simultaneously.
\end{proof}

\subsection{Algorithm in the Interval Setting}\label{sec:idec-algorithm}

Algorithm~\ref{algo-2} separates fitting, offline calibration and deployment.
The interval width represents a declared sensitivity budget or a calibrated
reference approximation budget. The region level $\gamma$ is applied after
choosing that width; the robustness radius $d^\star$ is an optional diagnostic
of a fixed candidate region.

\begin{algorithm}[t]
\caption{Interval Deep Evidential Classification with a declared width policy}
\label{algo-2}
\begin{algorithmic}[1]
\Require Fitted single evidential predictor $\pi(x)\in\Delta^{k-1}$
with $\pi_j(x)>0$, for every $j$; $\gamma\in(0,1)$;
$\epsilon\in(0,1-1/k)$; fixed radius, expected risk or high probability calibrated mode; in calibrated mode, frozen reference,
independent calibration inputs and budget $\zeta$; in high probability mode, fixed grid $\mathcal G$ and confidence parameter $\delta$
\Ensure A prediction region or an abstention diagnostic, and the declared width
\Statex \textbf{Offline:}
\If{fixed radius mode}
\State Specify $\rho\geq0$ independently of final evaluation outcomes.
\Else
\State Freeze the reference predictor, single predictor, and all tuning choices.
\State Fix $\zeta$ before inspecting the independent calibration inputs $X_1,\ldots,X_m$.
\If{high probability calibrated mode}
\State Select the smallest grid radius with $U_\rho\leq\zeta$ as in
Appendix~\ref{app:hp-calibration}; use $\rho=\infty$ if none exists.
\ElsIf{$\zeta\leq1/(m+1)$}
\State Set $\rho=\infty$.
\Else
\State Compute $D_\rho(X_i)$ using \eqref{eq:reference-distance} and \eqref{eq:tv-distance}.
\If{the calibration criterion in \eqref{eq:radius-crc}
is satisfied at $\rho=0$}
\State Set $\rho=0$.
\Else
\State Set $\rho_{\mathrm{lo}}=0$ and
$\rho_{\mathrm{hi}}=1$; double $\rho_{\mathrm{hi}}$
until the calibration criterion is satisfied.
\State Repeatedly test the midpoint of
$[\rho_{\mathrm{lo}},\rho_{\mathrm{hi}}]$:
if the criterion is satisfied, replace the upper
endpoint by the midpoint; otherwise, replace the
lower endpoint.
\State Stop after the prescribed number of iterations
and set $\rho=\rho_{\mathrm{hi}}$.
\EndIf
\State Record the computed residual
$\bigl(1+\sum_{i=1}^mD_\rho(X_i)\bigr)/(m+1)-\zeta$.
\EndIf
\EndIf
\Statex \textbf{For a new input $\tilde{x}$:}
\State Evaluate $\pi=\pi(\tilde{x})$ and form
$\mathcal K_\rho(\pi)$.
\State Compute $\mathrm{AU}_G$, $\mathrm{TU}_G$ and $\mathrm{EU}_G=\mathrm{TU}_G-\mathrm{AU}_G$
using Proposition~\ref{prop:gini-computation}.
\If{$1-1/k-\mathrm{TU}_G\geq\epsilon$}
\State Return the labels in decreasing order of $\pi_j$
up to the first cumulative sum reaching the threshold
in \eqref{eq-obj}, resolving ties by a fixed label ordering;
use $\mathcal Y$ at $\rho=\infty$.
\ElsIf{$2\mathrm{AU}_G\geq \mathrm{TU}_G$}
\State Abstain: total uncertainty is excessive and AU is at least EU.
\Else
\State Abstain: total uncertainty is excessive and EU exceeds AU.
\EndIf
\end{algorithmic}
\end{algorithm}

All finite radius calculations use $r=e^{-\rho}$ to avoid unnecessary overflow.
For prediction regions, \eqref{eq-obj} is equivalent to
$\sum_{j\notin R}\pi_j
\leq
\frac{\gamma r}{1-\gamma+\gamma r}$. 
We compute the probability outside $R$ by summing the
excluded class probabilities directly. This avoids
the loss of precision that can occur when subtracting
a value close to one from one. At $\rho=\infty$,
the interval is the full simplex and the Gini extrema use the explicit endpoint
in Proposition~\ref{prop:gini-computation}. At $\rho=\infty$, total Gini uncertainty reaches its
maximum, $\mathrm{TU}_G=1-1/k$, so
Algorithm~\ref{algo-2} abstains whenever $\epsilon>0$.
The candidate IHDR is then the full label set
$\mathcal Y$. Its coverage and size can still be
recorded when evaluating candidate regions, although
the algorithm returns an abstention.

The comparison between AU and EU uses
$\mathrm{EU}_G=\mathrm{TU}_G-\mathrm{AU}_G$:
$2\mathrm{AU}_G\geq\mathrm{TU}_G$ is equivalent to
$\mathrm{AU}_G\geq\mathrm{EU}_G$; otherwise,
$\mathrm{EU}_G>\mathrm{AU}_G$.
Numerical extrema are evaluated to a reported tolerance. When upper and lower numerical bounds are available, they certify
a threshold comparison if both bounds give the same decision.
Appendix~\ref{app:mnist-protocol} specifies the numerical comparisons
used in the reported experiments.
For expected risk calibration, we return a radius
whose computed criterion in \eqref{eq:radius-crc}
is at most $\zeta$.
We report numerical tolerances and compare the
calculated distances with results from an independent
optimisation solver; these checks assess numerical accuracy.
The theorem uses exact distances $D_\rho(X_i)$;
we do not provide a proved bound on rounding errors.

We recompute the uncertainty measures and prediction
regions whenever the chosen radius $\rho$ changes.
Probability scores evaluated at the single network
predictive vector $\pi(x)$ remain unchanged,
since selecting the width leaves that vector
unchanged. The reference model is evaluated during
calibration; subsequent predictions use the single
evidential network and the selected width.

With $m$ calibration inputs and $J$ reference generators, sorting ratios can be
cached in $O(mJk\log k)$ work. After sorting, evaluating the calibration criterion
\eqref{eq:radius-crc} at one candidate radius costs
$O(mJk^2)$ operations. Testing $T$ candidate radii
therefore costs $O(TmJk^2)$; these operation counts exclude the cost of obtaining
the calibration predictions.

For each new input, IDEC evaluates the single evidential
network once. Computing the Gini uncertainty measures
and the prediction region then requires
$O(k\log k+kN_G)$ operations, where $N_G$ is the number
of iterations used to compute $\mathrm{TU}_G$ numerically.

\section{Experiments}\label{sec:experiments}

We evaluate CDEC and IDEC on MNIST using separate observations for fitting,
model selection, calibration and testing. We assess their classification
accuracy, uncertainty ranking, prediction region size and computational
cost. For calibrated IDEC, we also examine how closely it approximates
a fitted reference credal set and how this affects the size of its
prediction regions. Additional tests assess performance under image
corruption.

\subsection{Design and Reproducibility}\label{sec:empirical-design}

A recorded random permutation divides the official $60000$ training
observations into $50000$ fitting, $5000$ tuning and $5000$ calibration
observations; the official $10000$ test observations are evaluated after
the models and calibration radii have been fixed. Pixels are divided by
$255$, and images are flattened. The study uses no pretrained features,
data augmentation or preprocessing fitted on held out observations; the archive checksum and split indices are recorded. Calibration observations
are excluded from fitting and model selection. The guarantee in Theorem~\ref{thm:radius-crc}
applies when the calibration inputs and a new test
input are exchangeable, conditional on the fitted
predictors.

We repeat the experiment three times using the same data split
and different random seeds for network initialisation and training.
For both CDEC and the CBDL comparator, each repetition uses
three networks, with LeakyReLU, hyperbolic tangent and SiLU
activations, respectively. Their encoder dimensions are
$784$--$64$--$64$--$6$.
CBDL-A3 uses one common Gaussian prior over weights and biases, and fits
a fully factorised Gaussian variational posterior for each of these three
likelihood models. Its training objective includes the full summed KL
term divided by the number of fitting observations. Prediction averages
probabilities over sampled network weights within each member, then forms
the convex hull of the three resulting predictive vectors. The primary
sample count is $M=20$, with $M=5$ and $M=100$ as sensitivity checks.
A fixed bank of weight samples is reused across calibration and test
inputs; the three sample counts use prefixes of the same bank. Calibration
therefore concerns this fixed numerical approximation to the Bayesian
predictive model.

CDEC-A3 uses the same encoder widths and activation choices, with latent
batch normalisation and a six layer radial flow for each class.
These choices match the encoder dimensions while retaining the methods'
different output models, parameter counts and training objectives.
The first CDEC member supplies the single PostNet predictor used by IDEC.
CDEC-A3 and CBDL-A3 use Shannon entropy, the same rule for predicting a single class, and exhaustive enumeration for minimum cardinality prediction
regions. IDEC uses the Gini uncertainty measures and retains the single
PostNet's point probabilities. Analytic entropy certificates and numerical
entropy optimisation are evaluated separately. The Bayesian comparator is
an independent implementation of CBDL with one prior and three likelihood
models \citep{ibnn}.

We also include a matched softmax ensemble, CreDE-3
\citep{wang2024crede}, and CREDIT \citep{wang2026credit} distilled from
either the softmax or CDEC ensemble. CREDIT-CDEC3 denotes the adaptation
using CDEC as its teacher. CD-CI \citep{huang2025distilling} uses the mean of the three
CDEC predictive vectors as its reference and the same single
PostNet as its student. At each input, it retains all probability
vectors whose divergence from the student's predictive vector
is at most a calibrated threshold. The main comparison uses
the $\alpha$-divergence defined in
Appendix~\ref{app:mnist-protocol}, with $\alpha=0.9$;
TV and KL divergences provide additional controls.
The threshold is selected using separate calibration inputs,
targeting a $90\%$ probability that the resulting credal set
contains the reference predictive vector at a new input. IDEC calibration instead controls expected approximation
error for the entire reference credal set, with primary budget
$\zeta=0.02$ and fixed sensitivity budgets $0.01$ and $0.05$.

Primary model configuration and checkpoint selection use tuning observations only.
The protocol fixes the calibration rules and budgets before calibration
and test evaluation, including the optional high probability rule.
The recorded study comprises 56 distinct benchmark fits, including
configuration candidates, of which 42 are selected. A separate training
repeat checks reproducibility. Appendix~\ref{app:mnist-protocol} gives
the seeds, objectives, numerical settings, validation checks and cost
accounting.

\subsection{Predictive and Set-Valued Performance}\label{sec:predictive-performance}

Table~\ref{tab:mnist-primary} reports point predictions and candidate
regions computed for all test inputs, before abstention. Regions satisfy
the $1-\gamma=0.95$ lower probability requirement under the corresponding
predictive model. Their empirical coverage is the proportion containing
the observed test label.
In this comparison, CDEC-A3 combines higher classification
accuracy with smaller prediction regions than CBDL-A3.
Relative to averaging the same three PostNet predictions,
CDEC achieves the same reported accuracy while requiring
each region to have probability at least $0.95$ under
every predictive vector. The mean region size increases
from $1.205$ to $1.340$.
IDEC preserves the classification accuracy of its single
PostNet predictor while using the interval width to
determine the predictive distributions represented by
its credal set. We assess the resulting increase in
region size together with reference approximation and
deployment cost in Sections~\ref{sec:radius-results}
and~\ref{sec:cost}.

\begin{table}[h!]
\centering
\small
\begin{tabular}{lrrr}
\toprule
Method & Accuracy (\%) & Mean size & Coverage (\%) \\
\midrule
CBDL-A3 & $97.50\pm0.15$ & 1.617 & 99.92 \\
CDEC-A3 & $98.11\pm0.12$ & 1.340 & 99.82 \\
PostNet-3 mean & $98.11\pm0.17$ & 1.205 & 99.74 \\
PostNet-single & $97.27\pm0.11$ & 1.160 & 99.12 \\
IDEC-rob & $97.27\pm0.11$ & 1.160 & 99.12 \\
IDEC-CRC(C), $\zeta=.02$ & $97.27\pm0.11$ & 2.133 & 99.82 \\
IDEC-HP(C), $\zeta=.02$ & $97.27\pm0.11$ & 3.217 & 99.92 \\
IDEC-CRC(B), $\zeta=.02$ & $97.27\pm0.11$ & 5.399 & 99.97 \\
CD-CI, $\alpha=.9$ & $97.27\pm0.11$ & 1.578 & 99.64 \\
CREDIT-DE3 & $97.29\pm0.08$ & 2.100 & 99.82 \\
CREDIT-CDEC3 & $97.50\pm0.11$ & 2.760 & 99.73 \\
CreDE-3 & $97.70\pm0.18$ & 1.163 & 99.54 \\
Softmax-3 & $97.70\pm0.10$ & 1.157 & 99.48 \\
\bottomrule
\end{tabular}
\caption{Clean MNIST candidate predictions and regions at nominal coverage $0.95$, before abstention. Accuracy variability is standard deviation across three initialisation blocks sharing the data partition. C and B indicate CDEC and CBDL calibration references. The single PostNet and every IDEC policy share their point predictor.}
\label{tab:mnist-primary}
\end{table}

The CDEC-minus-CBDL accuracy difference is $0.61$ percentage points,
with a paired $95\%$ test-case bootstrap interval $[0.42,0.81]$.
This interval conditions on the three fitted runs and the shared
data partition. The standard deviations in the table describe variation
across initialisations. 
Probability scores, including the Brier score,
are evaluated only for specified probability vectors, as detailed in
Appendix~\ref{app:mnist-protocol}.

\subsection{Independent Reference Calibration}\label{sec:radius-results}

We use IDEC-CRC to denote IDEC with the expected risk
calibration rule in \eqref{eq:radius-crc}, and IDEC-HP
for IDEC with the high probability calibration rule
in Appendix~\ref{app:hp-calibration}. The suffixes C and B in Table~\ref{tab:mnist-risk} indicate CDEC-A3
and CBDL-A3 references, respectively. IDEC-rob uses the fixed region
robustness radius $d^\star$ as its width; fixed $d=1$ provides a further
control.

Table~\ref{tab:mnist-risk} shows that reference calibration
reduces IDEC's approximation error relative to the CDEC
predictive credal set. At $\zeta=0.02$, IDEC-CRC(C) has
mean test approximation error $0.0233$, with conditional
test case interval $[0.0218,0.0248]$. The corresponding
errors are $0.0520$ for the single PostNet predictor,
$0.0441$ for IDEC-rob and $0.0448$ for fixed $d=1$.
By Proposition~\ref{prop:risk-duality}, a smaller approximation
error gives a tighter bound on how much IDEC can underestimate
the reference model's worst case expected loss, for losses
taking values in $[0,1]$.
This observed test average exceeds the calibration budget.
Theorem~\ref{thm:radius-crc} bounds an expectation over both the
calibration sample and a new exchangeable input; it does not require
the population error to be below the budget for every realised
calibration sample. Applying the theorem to the test set requires the
calibration inputs and a new test input to be exchangeable,
conditional on the fitted predictors.

\begin{table}[h!]
\centering
\small
\begin{tabular}{lrrr}
\toprule
Method & $D^{\rm CDEC}$ & $D^{\rm CBDL}$ & Mean size \\
\midrule
PostNet-single & 0.0520 & 0.0863 & 1.160 \\
IDEC-rob & 0.0441 & 0.0720 & 1.160 \\
IDEC, fixed $d=1$ & 0.0448 & 0.0793 & 1.266 \\
IDEC-CRC(C), $\zeta=.02$ & 0.0233 & 0.0524 & 2.133 \\
IDEC-HP(C), $\zeta=.02$ & 0.0140 & 0.0358 & 3.217 \\
IDEC-CRC(B), $\zeta=.02$ & 0.0065 & 0.0182 & 5.399 \\
CD-CI, $\alpha=.9$ & 0.0397 & 0.0677 & 1.578 \\
CREDIT-CDEC3 & 0.0472 & 0.0668 & 2.760 \\
CreDE-3 & 0.0575 & 0.0731 & 1.163 \\
TV-CRC(C), $\zeta=.02$ & 0.0225 & 0.0235 & 10.000 \\
\bottomrule
\end{tabular}
\caption{Mean directed reference discrepancy on the official test set. The primary expected-risk budget is $0.02$. CD-CI targets $90\%$ containment of the reference mean, and CREDIT/CreDE have different learning objectives; their discrepancies are descriptive common diagnostics. The high probability mode uses confidence $0.95$.}
\label{tab:mnist-risk}
\end{table}

A prespecified repeated calibration analysis uses 20 overlapping
$1000$ observation subsets of the calibration pool for each fitted run.
A later diagnostic evaluates each selected radius on the remaining
$4000$ observations in that pool. At budget $0.02$, the mean
CDEC-reference error is $0.0193$ on these complements and $0.0227$
on the official test set. The diagnostic provides a further evaluation of each
calibrated radius on observations unused in its selection.
It compares approximation error on the remaining calibration
observations and the test set, using the same fitted models
and selected radius. The repetitions describe variation
across overlapping subsets of a single calibration pool.

The high probability rule in Appendix~\ref{app:hp-calibration},
with confidence level $0.95$, gives mean test approximation
error $0.0140$ at $\zeta=0.02$. Its mean region size is
$3.217$, compared with $2.133$ for the expected risk rule
in \eqref{eq:radius-crc}. Thus, the stronger
calibration criterion reduces measured reference error while enlarging
prediction regions. With CBDL-A3 as reference, expected risk calibration
gives error $0.0182$ and mean size $5.399$. The approximation required
from the single predictor therefore depends substantially on the chosen
reference.

Figure~\ref{fig:mnist-risk-size} shows how the calibration
budget affects approximation error and prediction region
size for the CDEC reference. At $\zeta=0.02$, the TV ball
control has mean error $0.0225$ and returns all ten labels,
whereas IDEC-CRC(C) has mean error $0.0233$ and mean region
size $2.133$. At these measured error levels, the
multiplicative interval therefore yields substantially
smaller prediction regions. This comparison illustrates
how the choice of credal geometry affects the informativeness
of the regions obtained through reference calibration.

\begin{figure}[h!]
\centering
\includegraphics[width=.7\textwidth]{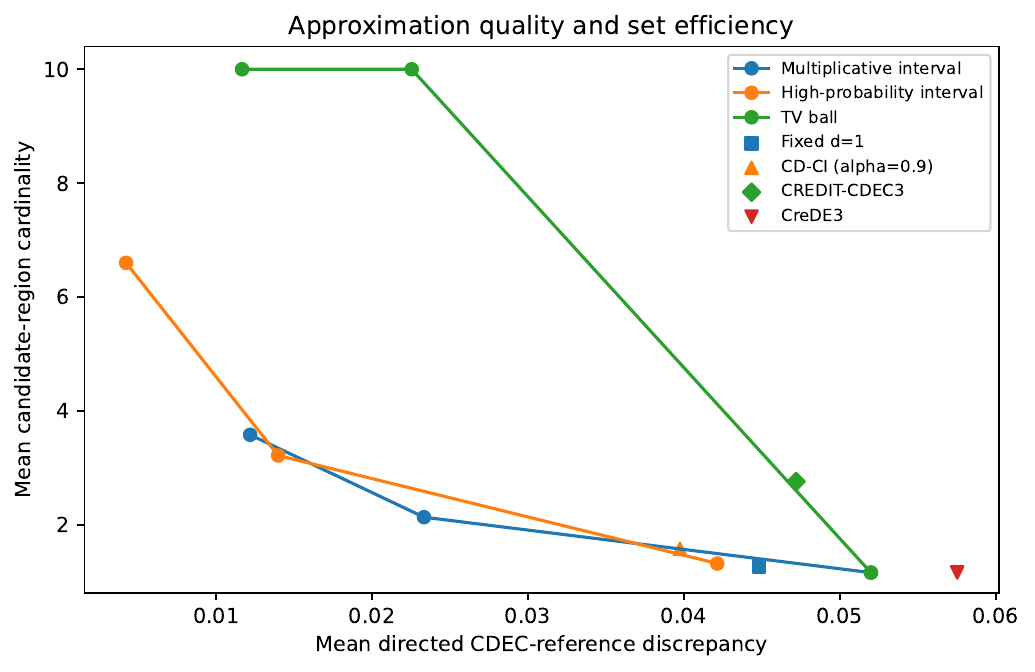}
\caption{Mean approximation error relative to CDEC-A3 and mean prediction
region size on clean MNIST. Each curve varies the calibration budget over
$\{0.01,0.02,0.05\}$. The additional points show fixed width IDEC,
CD-CI, CREDIT-CDEC3 and CreDE-3, which use their own fitting or calibration
criteria. All values are averaged across the three fitted runs on the
same official test inputs.}
\label{fig:mnist-risk-size}
\end{figure}

At $\zeta=0.05$, CDEC reference calibration selects zero width in
all three runs. The resulting singleton credal sets have Gini EU equal
to zero. 
CD-CI contains the reference mean on $89.81\%$ of test inputs, against
its $90\%$ containment target. Its error relative to the entire CDEC
set is $0.0397$, which measures a different property. CREDIT-CDEC3
and CreDE-3 have corresponding errors $0.0472$ and $0.0575$; their training objectives impose neither the CD-CI containment target
nor the IDEC reference-risk target. These comparisons assess approximation
and region size while preserving the methods' different objectives.

\subsection{Uncertainty Ranking and Robustness}\label{sec:uncertainty-estimation}
\label{sec:ood}

On clean test inputs, we evaluate how well EU scores rank each method's
incorrect predictions above its correct predictions. The resulting
AUROCs are $0.947$ for CDEC-A3, $0.868$ for CBDL-A3, $0.945$ for
IDEC-CRC(C) at $\zeta=0.02$, and $0.437$ for IDEC-rob.
The two IDEC variants share their point predictor and Gini uncertainty
definition, so their comparison isolates the effect of width selection.
Standard Shannon ensemble disagreement scores give AUROCs $0.961$
for the three PostNet predictive vectors and $0.957$ for the softmax
ensemble; the single PostNet's Dirichlet mutual information score gives
$0.926$. Across methods, both the predictions and the uncertainty
definitions can differ. 

The scalar and coordinate-wise Dirichlet credibility controls are
evaluated on the same prespecified subset of $1024$ test inputs.
At credibility level $90\%$, their mean region sizes are $2.223$ and
$1.311$, respectively. These controls concern probability mass under
the fitted Dirichlet distribution; they have their own interpretation,
separate from label coverage and reference risk calibration.

CDEC-A3 achieves higher mean classification accuracy
than CBDL-A3 under both translation and occlusion. Under Gaussian noise with standard
deviation $0.2$, followed by clipping to $[0,1]$, its accuracy falls
to $47.92\%$, compared with $90.30\%$ for CBDL-A3. Its uncertainty
scores nevertheless distinguish corrupted from clean images well.
This corruption detection task differs from ranking classification
errors within corrupted images. Appendix~\ref{app:mnist-corruption}
gives the complete protocol and results. 

For abstention with Shannon margin $\epsilon=0.1\log_2 k$, CDEC's
analytic certificate retains $95.60\%$ of clean inputs, with error
about $0.79\%$ among retained predictions. Numerical entropy
optimisation retains $99.96\%$. IDEC-CRC(C), with Gini margin
$\epsilon=0.1(1-1/k)$, retains $96.36\%$, with error $1.39\%$
among retained predictions. The CDEC certificate and IDEC-CRC(C) rules therefore retain,
on average, more than $95\%$ of the inputs while reducing
classification error among the retained predictions.
Before abstention, the corresponding mean errors are
$1.89\%$ for CDEC-A3 and $2.73\%$ for IDEC
(Table~\ref{tab:mnist-primary}). These comparisons demonstrate
the practical value of using uncertainty to decide when
to return a prediction. IDEC obtains these results using
a single evidential predictor at deployment, whose measured
computational cost is reported in Table~\ref{tab:mnist-time}; the numerical comparisons and their
tolerances are specified in Appendix~\ref{app:mnist-protocol}.
We report region coverage both across all inputs and among retained
outputs. The reference coverage conclusion in
Corollary~\ref{cor:reference-transfer} applies marginally across inputs.

\subsection{Computation, Optimisation and Modelling Costs}\label{sec:cost}

Table~\ref{tab:mnist-time} reports the measured fitting
costs and the time required to compute predictions,
uncertainty measures and prediction regions.

\begin{table}[h!]
\centering
\small
\begin{tabular}{lrrrr}
\toprule
Method & Fit (s) & Predict (ms) & Regions (ms) & Native UQ (ms) \\
\midrule
CBDL-A3 & 1105.2 & 14.59 & 16.06 & 99.81 \\
CDEC-A3 & 182.4 & 3.16 & 4.35 & 89.88 \\
IDEC-CRC(C), $\zeta=.02$ & 53.9 & 1.19 & 1.45 & 3.05 \\
CREDIT-DE3 & 65.7 & 0.46 & 1.34 & 13.28 \\
CREDIT-CDEC3 & 50.3 & 0.49 & 1.41 & 12.12 \\
CreDE-3 & 63.1 & 1.87 & 2.85 & 15.58 \\
CD-CI, $\alpha=.9$ & 53.9 & 1.17 & 11.58 & -- \\
\bottomrule
\end{tabular}
\caption{Mean full executed cost of selected required fits, excluding configuration search, and pooled median warm latency for a batch of 128 inputs. Region and native UQ timings include predictive generation. Bayesian prediction uses $M=20$ weight samples from a fixed
collection; the time required to generate this collection
is reported separately. CREDIT fitting times cover the
student network, with teacher training and prediction costs
reported separately. CD-CI region timing also includes its
intersection probability summary. The uncertainty timings use Shannon entropy
for CBDL-A3, CDEC-A3, CREDIT and CreDE-3, and Gini uncertainty
for IDEC.}
\label{tab:mnist-time}
\end{table}

Under the executed optimisation schedules, the selected CBDL fits
require a mean of $1105$ CPU-seconds per three member run, compared
with $182$ for CDEC. Fitting the single member used by IDEC requires
$54$ CPU-seconds; this member is already included in the CDEC cost.
Calibrated IDEC additionally requires its reference and offline
calibration calculations. The 56 benchmark fits, including rejected configuration
candidates, require approximately $6138$ CPU-seconds. Teacher target
generation adds $4.91$ seconds, counting each cache once, and the
separate reproducibility fit adds $47.94$ seconds. These sums differ
from elapsed time when fits run concurrently.

For batches of 128 inputs and $M=20$, generating predictive
probabilities takes $14.59$ ms for CBDL-A3 and $3.16$ ms for CDEC-A3.
Including their uncertainty calculations and prediction regions
narrows the difference to $99.81$ and $89.88$ ms, respectively.
IDEC-CRC(C) requires $3.05$ ms for predictive generation, Gini
uncertainty and regions. 
Computing predictions, uncertainty and prediction regions
takes $3.05$ ms for IDEC using Gini uncertainty, compared
with $12.12$--$13.28$ ms for the CREDIT variants using
Shannon entropy.

Timing uses one CPU thread, three warmup calls and nine repetitions
per fitted run, with interleaved method order. The reported medians
pool these repetitions for the same batch of the first 128 test
images. They measure latency for that batch, rather than variation
across different input batches. Peak process memory includes
libraries, data and validation; parameter storage is reported
separately. Teacher training and fitting/tuning predictions form
part of CREDIT's offline requirements, while reference evaluation
on calibration inputs forms part of calibrated IDEC's requirements.
The single predictor deployment cost therefore describes the
online stage of a larger procedure.

Three of the nine selected Bayesian checkpoints are early hyperbolic tangent 
iterates, while most other members use late checkpoints. Selecting
checkpoints by an estimate of the variational training objective
instead gives mean accuracy
$97.44\%$, compared with $97.50\%$ for the primary selection, and
changes EU error ranking AUROC from $0.868$ to $0.833$.
Posterior sample counts $M=5$ and $M=100$ give accuracies $97.42\%$
and $97.55\%$. The recorded histories and sensitivity results expose
these optimisation dependencies. 

Together, the MNIST experiments show that CDEC combines
higher clean classification accuracy with lower costs for
generating class probabilities than CBDL-A3 in this comparison.
For IDEC, they quantify how reference calibration improves
approximation and affects prediction region size, while
retaining deployment through a single evidential predictor.

\section{Discussion and Conclusion}\label{concl}

We have introduced CDEC and IDEC to construct predictive credal
sets for classification without integrating over network weights
at deployment. CDEC combines the class probability vectors
produced by several fitted evidential networks and retains all
their convex mixtures. We establish the mathematical relationship
between this construction and Credal Bayesian Deep Learning (CBDL) with a single specified prior
and several likelihood models. Both construct predictive credal
sets from alternative posterior predictives; the evidential
formulation makes each predictive vector available analytically.
The MNIST comparison examines the performance and computational
cost of these separately fitted constructions.

IDEC obtains its predictive credal set from one evidential
network by allowing its class probabilities to vary within an
interval controlled by $d$. We distinguish choosing this width
from computing $d^\star$, the largest perturbation under which
a fixed prediction region retains its required lower probability.
Our counterexamples show that requiring the interval
to contain reference predictions or a fixed proportion
of Dirichlet probability mass can produce substantial
imprecision because of relative changes in very small
class probabilities.

Our calibrated construction selects the width $d$ by comparing
IDEC with a fitted CDEC or CBDL reference on separate calibration
inputs. The approximation criterion measures how much IDEC
can underestimate the reference model's worst case expected
loss, considering all losses taking values in $[0,1]$.
We derive explicit formulas that compute this error from the
reference predictive vectors. With predictors fixed before
calibration and exchangeable calibration and future inputs,
conformal risk control bounds its expectation over both the
calibration sample and a new input. An optional rule gives
a high probability bound on the calibrated predictor's
population approximation error under independent sampling.
The reference is evaluated during calibration; deployment
requires only the single evidential predictor and the selected
width. Our label coverage bound also accounts explicitly for
how closely the reference set approximates the true conditional
class probabilities $\pi^\star_j(x)$, $j \in \mathcal{Y}$.

The MNIST results demonstrate the practical consequences of
these constructions. CDEC achieves higher clean classification
accuracy and lower costs for generating class probabilities
than the matched CBDL comparator. Calibrated IDEC improves
reference approximation and the ranking of incorrect predictions
compared with using $d^\star$ as its width, while retaining
a single network at deployment. These improvements accompany
larger prediction regions. For the CDEC reference, the expected
risk rule gives a test approximation error slightly above its
chosen budget, whereas the high probability rule reduces the
measured error and returns larger regions. The comparisons
therefore show how the calibration criterion and reference
model affect the balance between approximation, region size
and computational cost.

The corruption experiments distinguish the ability to recognise
unreliable inputs from the ability to classify them accurately.
CDEC has higher accuracy than the Bayesian comparator under
translation and occlusion. Under clipped Gaussian noise, it
continues to distinguish corrupted from clean images through
its uncertainty scores, but its classification accuracy
deteriorates substantially. This finding identifies predictive
robustness as a concrete direction for further development.

Several questions follow from these results. Can we obtain smaller prediction regions by allowing
the perturbation size to depend on the input or differ
between classes, while preserving the calibration
guarantee at the same approximation error budget $\zeta$? How should the reference model be selected
and assessed so that reference calibration yields informative
bounds on actual label coverage? Finally, can training with
label preserving corruptions improve classification on
transformations not encountered during training, while
preserving useful uncertainty scores and low deployment cost?
The training objective in \eqref{eq:cdec-augmented-uce}
and the stability result in
Proposition~\ref{prop:cdec-generator-stability}
provide a starting point for investigating this question.


\bibliography{references}

\begin{thebibliography}{37}
\providecommand{\natexlab}[1]{#1}
\providecommand{\url}[1]{\texttt{#1}}
\expandafter\ifx\csname urlstyle\endcsname\relax
  \providecommand{\doi}[1]{doi: #1}\else
  \providecommand{\doi}{doi: \begingroup \urlstyle{rm}\Url}\fi

\bibitem[Abellán et~al.(2006)Abellán, Klir, and Moral]{abellan}
Joaquín Abellán, George~Jiří Klir, and Serafín Moral.
\newblock Disaggregated total uncertainty measure for credal sets.
\newblock \emph{International Journal of General Systems}, 1\penalty0
  (35):\penalty0 29--44, 2006.

\bibitem[Angelopoulos et~al.(2024)Angelopoulos, Bates, Fisch, Lei, and
  Schuster]{angelopoulos2024crc}
Anastasios~N. Angelopoulos, Stephen Bates, Adam Fisch, Lihua Lei, and Tal
  Schuster.
\newblock Conformal risk control.
\newblock In \emph{International Conference on Learning Representations}, 2024.
\newblock URL \url{https://arxiv.org/abs/2208.02814}.

\bibitem[Audenaert(2007)]{audenaert2007sharp}
Koenraad M.~R. Audenaert.
\newblock A sharp continuity estimate for the von {Neumann} entropy.
\newblock \emph{Journal of Physics A: Mathematical and Theoretical},
  40\penalty0 (28):\penalty0 8127--8136, 2007.
\newblock \doi{10.1088/1751-8113/40/28/S18}.

\bibitem[Augustin et~al.({2014})Augustin, Coolen, De~Cooman, and
  Troffaes]{augustin}
Thomas Augustin, Frank~P.A. Coolen, Gert De~Cooman, and Matthias~C.M. Troffaes
  (eds.).
\newblock \emph{{Introduction to imprecise probabilities}}.
\newblock {Wiley Series in Probability and Statistics}. {John Wiley and Sons},
  {2014}.

\bibitem[Bernard(2005)]{bernard_bom}
Jean-Marc Bernard.
\newblock An introduction to the imprecise {D}irichlet model for multinomial
  data.
\newblock \emph{International Journal of Approximate Reasoning}, 39\penalty0
  (2):\penalty0 123--150, 2005.
\newblock Imprecise Probabilities and Their Applications.

\bibitem[Blundell et~al.(2015)Blundell, Cornebise, Kavukcuoglu, and
  Wierstra]{blundell2015weight}
Charles Blundell, Julien Cornebise, Koray Kavukcuoglu, and Daan Wierstra.
\newblock Weight uncertainty in neural networks.
\newblock In \emph{International Conference on Machine Learning}, pp.\
  1613--1622, 2015.

\bibitem[Caprio(2026)]{caprio2025joyscategoricalconformalprediction}
Michele Caprio.
\newblock A category-theoretic analysis of conformal prediction, 2026.
\newblock URL \url{https://arxiv.org/abs/2507.04441}.

\bibitem[Caprio et~al.(2024)Caprio, Dutta, Jang, Lin, Ivanov, Sokolsky, and
  Lee]{ibnn}
Michele Caprio, Souradeep Dutta, Kuk~Jin Jang, Vivian Lin, Radoslav Ivanov,
  Oleg Sokolsky, and Insup Lee.
\newblock Credal bayesian deep learning.
\newblock \emph{Transactions on Machine Learning Research}, 2024.
\newblock ISSN 2835-8856.
\newblock URL \url{https://openreview.net/forum?id=4NHF9AC5ui}.

\bibitem[Caprio et~al.(2025)Caprio, Sale, and
  H\"ullermeier]{pmlr-v290-caprio25b}
Michele Caprio, Yusuf Sale, and Eyke H\"ullermeier.
\newblock Conformal prediction regions are imprecise highest density regions.
\newblock In Sébastien Destercke, Alexander Erreygers, Max Nendel, Frank
  Riedel, and Matthias C.~M. Troffaes (eds.), \emph{Proceedings of the
  Fourteenth International Symposium on Imprecise Probabilities: Theories and
  Applications}, volume 290 of \emph{Proceedings of Machine Learning Research},
  pp.\  47--59. PMLR, 15--18 Jul 2025.
\newblock URL \url{https://proceedings.mlr.press/v290/caprio25b.html}.

\bibitem[Cella \& Martin(2022)Cella and Martin]{CELLA20221}
Leonardo Cella and Ryan Martin.
\newblock Valid inferential models for prediction in supervised learning
  problems.
\newblock \emph{International Journal of Approximate Reasoning}, 150:\penalty0
  1--18, 2022.

\bibitem[Charpentier et~al.(2020)Charpentier, Z\"{u}gner, and
  G\"{u}nnemann]{charpentier}
Bertrand Charpentier, Daniel Z\"{u}gner, and Stephan G\"{u}nnemann.
\newblock Posterior network: Uncertainty estimation without {OOD} samples via
  density-based pseudo-counts.
\newblock In H.~Larochelle, M.~Ranzato, R.~Hadsell, M.F. Balcan, and H.~Lin
  (eds.), \emph{Advances in Neural Information Processing Systems}, volume~33,
  pp.\  1356--1367. Curran Associates, Inc., 2020.

\bibitem[Cohen \& Fausti(2023)Cohen and Fausti]{cohen2024hilbert}
Samuel~N. Cohen and Eliana Fausti.
\newblock Hyperbolic contractivity and the {Hilbert} metric on probability
  measures.
\newblock \emph{arXiv:2309.02413}, 2023.
\newblock URL \url{https://arxiv.org/abs/2309.02413}.
\newblock Version 2, revised 2024.

\bibitem[Coolen(1992)]{coolen}
Frank P.~A. Coolen.
\newblock Imprecise highest density regions related to intervals of measures.
\newblock \emph{Memorandum {COSOR}}, 9254, 1992.

\bibitem[Cozman(2000)]{cozman3}
Fabio~G. Cozman.
\newblock Credal networks.
\newblock \emph{Artificial Intelligence}, 120\penalty0 (2):\penalty0 199--233,
  2000.
\newblock ISSN 0004-3702.
\newblock \doi{https://doi.org/10.1016/S0004-3702(00)00029-1}.
\newblock URL
  \url{https://www.sciencedirect.com/science/article/pii/S0004370200000291}.

\bibitem[Dantzig(1963)]{dantzig}
George~B. Dantzig.
\newblock \emph{Linear Programming and Extensions}.
\newblock Princeton University Press, 1963.
\newblock URL \url{http://www.jstor.org/stable/j.ctt1cx3tvg}.

\bibitem[DeRobertis \& Hartigan(1981)DeRobertis and Hartigan]{lorraine}
Lorraine DeRobertis and J.~A. Hartigan.
\newblock {Bayesian Inference Using Intervals of Measures}.
\newblock \emph{The Annals of Statistics}, 9\penalty0 (2):\penalty0 235 -- 244,
  1981.

\bibitem[Hoff(2009)]{hoff}
Peter~D. Hoff.
\newblock \emph{A First Course in {B}ayesian Statistical Methods}.
\newblock New York : Springer, 2009.

\bibitem[Hofman et~al.(2024)Hofman, Sale, and
  H{\"u}llermeier]{hofman2024proper}
Paul Hofman, Yusuf Sale, and Eyke H{\"u}llermeier.
\newblock Quantifying aleatoric and epistemic uncertainty with proper scoring
  rules.
\newblock \emph{arXiv:2404.12215}, 2024.
\newblock URL \url{https://arxiv.org/abs/2404.12215}.

\bibitem[Huang et~al.(2025)Huang, Park, Paoletti, and
  Simeone]{huang2025distilling}
Jiayi Huang, Sangwoo Park, Nicola Paoletti, and Osvaldo Simeone.
\newblock Distilling calibration via conformalized credal inference.
\newblock In \emph{International Joint Conference on Neural Networks}, 2025.
\newblock URL \url{https://arxiv.org/abs/2501.06066}.

\bibitem[H{\"u}llermeier \& Waegeman(2021)H{\"u}llermeier and Waegeman]{eyke}
Eyke H{\"u}llermeier and Willem Waegeman.
\newblock Aleatoric and epistemic uncertainty in machine learning: an
  introduction to concepts and methods.
\newblock \emph{Machine Learning}, 110\penalty0 (3):\penalty0 457--506, 2021.
\newblock \doi{10.1007/s10994-021-05946-3}.

\bibitem[Javanmardi et~al.(2024)Javanmardi, Stutz, and
  H{\"u}llermeier]{javanmardi2024credal}
Alireza Javanmardi, David Stutz, and Eyke H{\"u}llermeier.
\newblock Conformalized credal set predictors.
\newblock In \emph{Advances in Neural Information Processing Systems},
  volume~37, 2024.
\newblock URL \url{https://arxiv.org/abs/2402.10723}.

\bibitem[Javanmardi et~al.(2026)Javanmardi, Zargarbashi, Thies, Waegeman,
  Bojchevski, and H{\"u}llermeier]{javanmardi2026optimal}
Alireza Javanmardi, Soroush~H. Zargarbashi, Santo M. A.~R. Thies, Willem
  Waegeman, Aleksandar Bojchevski, and Eyke H{\"u}llermeier.
\newblock Optimal conformal prediction under epistemic uncertainty.
\newblock In \emph{Proceedings of the 42nd Conference on Uncertainty in
  Artificial Intelligence}, volume 337 of \emph{Proceedings of Machine Learning
  Research}, pp.\  2461--2479, 2026.
\newblock URL \url{https://proceedings.mlr.press/v337/javanmardi26a.html}.

\bibitem[J{\"u}rgens et~al.(2024)J{\"u}rgens, Meinert, Bengs, H{\"u}llermeier,
  and Waegeman]{mira}
Mira J{\"u}rgens, Nis Meinert, Viktor Bengs, Eyke H{\"u}llermeier, and Willem
  Waegeman.
\newblock Is epistemic uncertainty faithfully represented by evidential deep
  learning methods?
\newblock In \emph{Proceedings of the 41st International Conference on Machine
  Learning}, volume 235 of \emph{Proceedings of Machine Learning Research},
  pp.\  22624--22642, 2024.
\newblock URL \url{https://proceedings.mlr.press/v235/juergens24a.html}.

\bibitem[Levi(1980)]{levi2}
Isaac Levi.
\newblock \emph{The Enterprise of Knowledge}.
\newblock London, UK : MIT Press, 1980.

\bibitem[Maurer \& Pontil(2009)Maurer and Pontil]{maurer2009empirical}
Andreas Maurer and Massimiliano Pontil.
\newblock Empirical bernstein bounds and sample variance penalization.
\newblock In \emph{Proceedings of the Conference on Learning Theory}, 2009.
\newblock URL \url{https://www.cs.mcgill.ca/~colt2009/papers/012.pdf}.

\bibitem[Pedregosa et~al.(2011)Pedregosa, Varoquaux, Gramfort, Michel, Thirion,
  Grisel, Blondel, Prettenhofer, Weiss, Dubourg, Vanderplas, Passos,
  Cournapeau, Brucher, Perrot, and Duchesnay]{sklearn2011}
Fabian Pedregosa, Ga{\"e}l Varoquaux, Alexandre Gramfort, Vincent Michel,
  Bertrand Thirion, Olivier Grisel, Mathieu Blondel, Peter Prettenhofer, Ron
  Weiss, Vincent Dubourg, Jake Vanderplas, Alexandre Passos, David Cournapeau,
  Matthieu Brucher, Matthieu Perrot, and {\'E}douard Duchesnay.
\newblock Scikit-learn: Machine learning in {Python}.
\newblock \emph{Journal of Machine Learning Research}, 12:\penalty0 2825--2830,
  2011.

\bibitem[Rezende \& Mohamed(2015)Rezende and Mohamed]{rezende}
Danilo Rezende and Shakir Mohamed.
\newblock Variational inference with normalizing flows.
\newblock In Francis Bach and David Blei (eds.), \emph{Proceedings of the 32nd
  International Conference on Machine Learning}, volume~37 of \emph{Proceedings
  of Machine Learning Research}, pp.\  1530--1538, Lille, France, 2015. PMLR.

\bibitem[Senge et~al.(2014)Senge, Bösner, Dembczyński, Haasenritter, Hirsch,
  Donner-Banzhoff, and Hüllermeier]{senge_2014_ReliableClassificationLearning}
Robin Senge, Stefan Bösner, Krzysztof Dembczyński, Jörg Haasenritter, Oliver
  Hirsch, Norbert Donner-Banzhoff, and Eyke Hüllermeier.
\newblock Reliable classification: {{Learning}} classifiers that distinguish
  aleatoric and epistemic uncertainty.
\newblock \emph{Information Sciences}, 255:\penalty0 16--29, 2014.

\bibitem[Shafer \& Vovk(2008)Shafer and Vovk]{conformal_tutorial}
Glenn Shafer and Vladimir Vovk.
\newblock A tutorial on conformal prediction.
\newblock \emph{Journal of Machine Learning Research}, 9:\penalty0 371--421,
  2008.

\bibitem[Shi(2026)]{shi2026smoothing}
Xian Shi.
\newblock Probabilistic entanglement distillation and cost under approximately
  nonentangling and dually nonentangling instruments.
\newblock \emph{arXiv:2601.00383v1}, 2026.
\newblock URL \url{https://arxiv.org/abs/2601.00383v1}.
\newblock Version 1, Appendix A.7.

\bibitem[Ulmer et~al.(2023)Ulmer, Hardmeier, and Frellsen]{ulmer}
Dennis~Thomas Ulmer, Christian Hardmeier, and Jes Frellsen.
\newblock Prior and posterior networks: A survey on evidential deep learning
  methods for uncertainty estimation.
\newblock \emph{Transactions on Machine Learning Research}, 2023.
\newblock ISSN 2835-8856.
\newblock URL \url{https://openreview.net/forum?id=xqS8k9E75c}.

\bibitem[Varshney(2016)]{varshney2016engineering}
Kush~R Varshney.
\newblock Engineering safety in machine learning.
\newblock In \emph{2016 Information Theory and Applications Workshop (ITA)},
  pp.\  1--5. IEEE, 2016.

\bibitem[Walley(1991)]{walley}
Peter Walley.
\newblock \emph{{Statistical Reasoning with Imprecise Probabilities}},
  volume~42 of \emph{Monographs on Statistics and Applied Probability}.
\newblock London : Chapman and Hall, 1991.

\bibitem[Walley(1996)]{walley_marbles}
Peter Walley.
\newblock Inferences from multinomial data: Learning about a bag of marbles.
\newblock \emph{Journal of the Royal Statistical Society. Series B
  (Methodological)}, 58\penalty0 (1):\penalty0 3--57, 1996.

\bibitem[Wang et~al.(2024)Wang, Cuzzolin, Manchingal, Shariatmadar, Moens, and
  Hallez]{wang2024crede}
Kaizheng Wang, Fabio Cuzzolin, Shireen~Kudukkil Manchingal, Keivan
  Shariatmadar, David Moens, and Hans Hallez.
\newblock Credal deep ensembles for uncertainty quantification.
\newblock In \emph{Advances in Neural Information Processing Systems},
  volume~37, 2024.
\newblock URL
  \url{https://proceedings.neurips.cc/paper_files/paper/2024/hash/911fc798523e7d4c2e9587129fcf88fc-Abstract.html}.

\bibitem[Wang et~al.(2026)Wang, Cuzzolin, Moens, and Hallez]{wang2026credit}
Kaizheng Wang, Fabio Cuzzolin, David Moens, and Hans Hallez.
\newblock Credal ensemble distillation for uncertainty quantification.
\newblock In \emph{Proceedings of the AAAI Conference on Artificial
  Intelligence}, 2026.
\newblock URL \url{https://arxiv.org/abs/2511.13766}.
\newblock Extended version arXiv:2511.13766.

\bibitem[Zaffalon(2002)]{zaffalon}
Marco Zaffalon.
\newblock The naive credal classifier.
\newblock \emph{Journal of Statistical Planning and Inference}, 105\penalty0
  (1):\penalty0 5--21, 2002.
\newblock ISSN 0378-3758.
\newblock \doi{https://doi.org/10.1016/S0378-3758(01)00201-4}.
\newblock URL
  \url{https://www.sciencedirect.com/science/article/pii/S0378375801002014}.
\newblock Imprecise Probability Models and their Applications.

\end{thebibliography}
\bibliographystyle{tmlr}
\appendix
\section{Proofs}\label{proofs}

\subsection{Predictive Hulls and Entropy}
\begin{proof}[Proof of Proposition~\ref{prop:affine}]
For every simplex vector $w$, linearity of integration gives
\[
T_x\!\left(\sum_sw_sQ_s\right)(A)
=\sum_sw_s\int K_x(\theta,A)Q_s(\mathrm d\theta)
=\sum_sw_s(T_xQ_s)(A).
\]
As $w$ ranges over the simplex, both sides range over the claimed sets. For
model-specific kernels $K_{s,x}$, use the augmented posterior
$\tilde Q_s=\delta_s\otimes Q_s$ and the common kernel
$\tilde K_x((s,\theta),A)=K_{s,x}(\theta,A)$. This retains the intended
posterior--likelihood pairing. For the evidential model, integration of
$\pi_j$ under $\operatorname{Dir}(\alpha_s)$ gives
$\alpha_{s,j}/\sum_l\alpha_{s,l}$, establishing the final claim.
\end{proof}

\paragraph{One Prior and Mixtures of Likelihood Factors.}
Let $\Pi_0$ be a common prior and $L_s$ nonnegative measurable factors with
$Z_s=\int L_s\,\mathrm d\Pi_0\in(0,\infty)$. Write
$Q_s(\mathrm d\theta)=L_s(\theta)\Pi_0(\mathrm d\theta)/Z_s$.
For $\lambda\in\Delta^{S-1}$, the posterior associated with
$L_\lambda=\sum_s\lambda_sL_s$ is
\[
Q_\lambda
=\sum_s\frac{\lambda_sZ_s}{\sum_t\lambda_tZ_t}\,Q_s.
\]
Conversely, any posterior weights $w$ are obtained by taking
$\lambda_s=(w_s/Z_s)/\sum_t(w_t/Z_t)$. Thus normalised mixtures of the factors
generate exactly the posterior hull. This identity concerns the declared prior
and factors, including input-dependent virtual factors in the evidential
construction. It does not identify them with a separately fitted weight-space
Bayesian model.

\begin{proof}[Proof of Theorem~\ref{thm-imp-reduced}]
For $P_\beta=\sum_s\beta_sP_s$, entropy concavity gives
\[
H(P_\beta)\geq\sum_s\beta_sH(P_s)\geq\min_sH(P_s).
\]
A generator attaining the last minimum belongs to the hull, so the lower entropy
is $A$. Every generator also belongs to the hull, giving
$\mathrm{TU}\geq L$. For the upper bound introduce
$J\sim\operatorname{Cat}(\beta)$ and $X\mid J=s\sim P_s$. Then
\[
H(X)\leq H(X,J)=H(\beta)+\sum_s\beta_sH(P_s)
\leq\log_2S+L.
\]
The universal bound $H(X)\leq\log_2k$ yields $U$. Subtracting the exact lower
entropy proves the EU inequalities. No independence between the component
distributions is assumed or needed.
\end{proof}

\paragraph{Numerical Upper Entropy Certificate.}
For positive generators, $f(\beta)=H(\sum_s\beta_sp_s)$ is differentiable and
concave on the simplex. Its supporting hyperplane satisfies
$f(w)\leq f(\beta)+\nabla f(\beta)^\top(w-\beta)$.
Maximising the linear term over $w$ gives \eqref{eq:entropy-gap}.
For binary predictions, the first-coordinate image of the hull is $[a,b]$.
Binary entropy increases up to $1/2$ and decreases afterwards; its maximum
is the projection of $1/2$ onto that interval and its minimum is at an endpoint.
This proves Remark~\ref{rem:binary}.

\subsection{Intervals and Gini Uncertainty}
\begin{proof}[Proofs of Proposition~\ref{prop-cred-int} and
Corollary~\ref{cor-equiv}]
Let $B=\prod_j[\ell_j,u_j]$ and $N(h)=h/\sum_jh_j$. The positive lower endpoints
make $N$ continuous throughout the compact box, so $N(B)$ is compact.
If $h=\sum_v\lambda_vv$ is a convex combination of vertices, put
$t_v=\sum_jv_j$ and $t_h=\sum_jh_j$. Then
\[
N(h)=\sum_v\frac{\lambda_vt_v}{t_h}N(v),
\]
with nonnegative weights summing to one. Conversely, for
$p=\sum_v\mu_vN(v)$ put
$c=\sum_v\mu_v/t_v$, $\lambda_v=(\mu_v/t_v)/c$, and
$h=\sum_v\lambda_vv$. Then $h\in B$, $\sum_jh_j=1/c$ and $N(h)=p$.
Thus $N(B)$ equals the convex hull of its normalised vertices and is convex.
Specialising $\ell=p,u=(1+d)p$ proves the corollary.

For a nonempty proper $A$, the ratio
$h(A)/(h(A)+h(A^c))$ is increasing in $h(A)$ and decreasing in $h(A^c)$.
Its minimum uses $h_j=\ell_j$ on $A$ and $h_j=u_j$ on $A^c$; the maximum
uses the reverse assignment. This gives the displayed envelope formulas.
For the empty and full events the same formulas give zero and one, respectively.
\end{proof}

\begin{proof}[Proof of Proposition~\ref{prop:gini-computation}]
Put $c=e^\rho$ and order $p_{(1)}\geq\cdots\geq p_{(k)}>0$.
Minimising $G$ over the normalised box is equivalent to maximising
\[
\frac{\sum_ih_i^2}{(\sum_ih_i)^2},\qquad p_i\leq h_i\leq cp_i.
\]
Fix the value of $\sum_ih_i$, and write $h_{(i)}=p_{(i)}+x_i$,
where $0\leq x_i\leq(c-1)p_{(i)}$. If $i<j$, the $i$th coordinate is not
full and $x_j>0$, transfer
$\delta=\min\{(c-1)p_{(i)}-x_i,x_j\}$ from $x_j$ to $x_i$.
When $\delta=x_j$, the increase in the sum of squares is
$2x_j(p_{(i)}+x_i-p_{(j)})\geq0$.
When $\delta=(c-1)p_{(i)}-x_i$, it is
$2\delta(cp_{(i)}-p_{(j)}-x_j)\geq0$ because
$p_{(j)}+x_j\leq cp_{(j)}\leq cp_{(i)}$.
Repeated transfers fill the coordinates in order, leaving at most one partial
coordinate. Hence an optimum for each fixed mass lies on a segment with the
first $m$ coordinates at $cp_{(i)}$, the next at a variable $x$, and the
remaining coordinates at $p_{(i)}$.

On such a segment the objective has the form
$f(x)=(A+x^2)/(B+x)^2$, with $A,B>0$ and
$f'(x)=2(Bx-A)/(B+x)^3$.
Its stationary point is a minimum, so a segment maximum occurs at an endpoint.
A global maximum therefore occurs among the $k+1$ prefix-inflated vertices.
Normalising those vectors and dividing numerator and denominator by $c^2$
gives \eqref{eq:gini-au}.

For upper Gini, if $q=N(h)$ set $\lambda=1/\sum_jh_j$.
The mass bounds give $r\leq\lambda\leq1$ and
$\lambda p_j\leq q_j\leq(\lambda/r)p_j$.
Conversely, any feasible $(q,\lambda)$ in \eqref{eq:gini-tu} yields
$h_j=q_j/\lambda\in[p_j,cp_j]$ and $N(h)=q$.
Thus the feasible probability set is exactly the interval. Maximising
$G(q)=1-\|q\|^2$ proves the quadratic-program formula.
The full-simplex endpoint has minimum Gini zero and maximum $1-1/k$.
\end{proof}

\paragraph{Equivalent Scalar Calculation for Upper Gini.}
Rescale the unnormalised box to $rp_j\leq h_j\leq p_j$. If
$r\max_jp_j\leq\min_jp_j$, a constant vector is feasible and normalises to the
uniform distribution, which maximises Gini. Otherwise put
\[
h_j(t)=\min\{p_j,\max\{rp_j,t\}\},\qquad
F(t)=\sum_jh_j(t)(h_j(t)-t).
\]
The function $F$ is continuous, nonincreasing and piecewise linear: free
coordinates contribute zero and fixed coordinates contribute $h_j^2-th_j$.
It is nonnegative at $t=\min_jp_j$ and nonpositive at $t=r\max_jp_j$.
Since the uniform vector is infeasible, every interval between breakpoints
contains a fixed positive coordinate, so the slope is negative there.
Bisection gives the root $t$.

At that root, $t=\|h\|^2/\sum_jh_j$ and the coordinate clipping conditions give
$(h-t\mathbf1)^\top(v-h)\geq0$ for every vector $v$ in the box.
Since $(h-t\mathbf1)^\top h=0$, we have
$h^\top v\geq t\sum_jv_j$. Cauchy--Schwarz implies
\[
\frac{\|v\|}{\sum_jv_j}\geq\frac{\|h\|}{\sum_jh_j}.
\]
Thus $N(h)$ minimises the normalised squared norm and maximises Gini globally.
This proves equivalence to the convex programme, independently of a solver's
success flag. Numerical tolerances and the explicit $r=0$ endpoint remain
implementation considerations.

\subsection{Robustness and Geometric Obstructions}
\begin{proof}[Proof of Proposition~\ref{prop:robust-radius}]
For $a\in[1-\gamma,1)$, solve
$a/[a+(1+d)(1-a)]\geq1-\gamma$ for $d$. This gives
$d\leq[a-(1-\gamma)]/[(1-\gamma)(1-a)]$.
The expression is nonnegative, and equality holds at its finite endpoint.
For $a=1$ the event has lower probability one for every width.
For $a<1-\gamma$, even $d=0$ fails, and nesting prevents a larger width from
repairing the same event. The equivalence \eqref{eq-obj} follows from solving
the same inequality for $a$. Among events of a fixed cardinality the largest
$p$-mass is attained by the top probabilities, proving the stated IHDR rule.
\end{proof}

\paragraph{Projective Representation.}
A probability vector $q$ lies in the normalised box exactly when some $t>0$
satisfies $tp_j\leq q_j\leq e^\rho tp_j$ for every $j$. Such a $t$ exists
exactly when $\max_j(q_j/p_j)/\min_j(q_j/p_j)\leq e^\rho$.
Taking logarithms proves \eqref{eq:hilbert}. This is the established Hilbert
metric representation \citep{cohen2024hilbert}.

\begin{proof}[Proof of Theorem~\ref{thm:geometry}]
For (i), the common-coordinate ratios $q_{t,j}/p_{t,j}$ are $1+t$ and the
rare-coordinate ratio is $t$. This gives $\rho_t=\log((1+t)/t)$.
The total variation is $t-t^2$. The ratio of the largest and smallest
coordinate ratios of the uniform distribution relative to $p_t$ is at most
$(1+t)/t$ for $0<t\leq1/2$, so the uniform distribution belongs to the interval.
Its Gini entropy is $2/3$. Inflating only the first unnormalised coordinate by
$(1+t)/t$ yields, after normalisation, a vector tending to $e_1$.
Thus lower Gini tends to zero. Both $p_t$ and $q_t$ tend to
$(1/2,1/2,0)$ uniformly over their convex combinations, so continuity of Gini
makes the reference hull's gap tend to zero.

For (ii), write $Z_1,Z_2\sim\operatorname{Gamma}(n,1)$ and
$Z_3\sim\operatorname{Gamma}(1,1)$ independently, and
$\Pi_n=Z/\sum_jZ_j$. Common normalisation cancels in the log-ratio range, so
\[
H_{\rm proj}(p_n,\Pi_n)
=\max\{\log(Z_1/n),\log(Z_2/n),\log Z_3\}
-\min\{\log(Z_1/n),\log(Z_2/n),\log Z_3\}.
\]
The first two terms tend to zero in probability and $Z_3$ is exponential.
The displayed range converges in distribution to $|\log E|$,
$E\sim\operatorname{Exp}(1)$. Its distribution function on $r\geq0$ is
$e^{-e^{-r}}-e^{-e^r}$, continuous and strictly increasing from zero to one.
Quantile convergence gives $\rho_{\eta,n}\to r_\eta>0$.

Let $c_n=e^{\rho_{\eta,n}}\to c=e^{r_\eta}$. The third coordinate of every
normalised box vector tends uniformly to zero. The attainable first coordinates
converge to $[1/(1+c),c/(1+c)]$, and every point of that interval is approached
by continuously varying the two common unnormalised coordinates. Hence the
credal sets converge to the corresponding segment on the two-class face.
Uniform continuity of Gini yields upper entropy $1/2$ and lower entropy
$2c/(1+c)^2$ in the limit. The Dirichlet mean tends to $(1/2,1/2,0)$ and
$\operatorname{tr}\operatorname{Cov}(\Pi_n)=G(p_n)/(2n+2)\to0$,
which also proves the stated concentration of $\Pi_n$.
\end{proof}

\subsection{Approximation Formulas and Decision Risk}

\begin{lemma}[Projection onto a bounded simplex]
\label{lem:box-projection}
Let $\pi,\pi^{\mathrm{ref}}\in\Delta^{k-1}$, with
$\pi_j>0$ for every $j$, and let $0<t\leq1\leq u$.
Define
\[
\mathcal B_{t,u}
=
\left\{
\tilde{\pi}\in\Delta^{k-1}:
t\pi_j\leq\tilde{\pi}_j\leq u\pi_j,\quad
j=1,\ldots,k
\right\}.
\]
Then
\[
\min_{\tilde{\pi}\in\mathcal B_{t,u}}
\operatorname{TV}(\pi^{\mathrm{ref}},\tilde{\pi})
=
\max\left\{
\sum_{j=1}^k(t\pi_j-\pi_j^{\mathrm{ref}})_+,
\sum_{j=1}^k(\pi_j^{\mathrm{ref}}-u\pi_j)_+
\right\}.
\]
\end{lemma}

\begin{proof}
Denote the two sums by $L$ and $U$, respectively.
Any feasible probability vector must add at least
$L$ mass to coordinates below their lower bounds
and remove at least $U$ mass from coordinates above
their upper bounds. For probability vectors, total
variation equals both the total positive deviation
and the total negative deviation. The distance is
therefore at least $\max\{L,U\}$.

Clip each coordinate of $\pi^{\mathrm{ref}}$ to
$[t\pi_j,u\pi_j]$. The resulting vector has total
mass $1+L-U$.

If $L>U$, remove $L-U$ mass while respecting the
lower bounds. The available capacity is sufficient
because the lower bounds sum to $t\leq1$.
Every coordinate from which mass can be removed is
at or below its original value
$\pi_j^{\mathrm{ref}}$, so this adjustment does not
increase the total positive deviation.
The final positive and negative deviations both
equal $L$.

If $U>L$, add $U-L$ mass while respecting the upper
bounds. Their sum is $u\geq1$, so sufficient capacity
is available. Every coordinate to which mass can be
added is at or above its original value
$\pi_j^{\mathrm{ref}}$. The final positive and
negative deviations both equal $U$.

If $L=U$, the clipped vector already sums to one.
In every case, the resulting probability vector
attains distance $\max\{L,U\}$.
\end{proof}

\begin{proof}[Proof of Proposition~\ref{prop:tv-formulas}]
Use the ordering and cumulative masses in the
proposition. Define
\[
L(t)
=
\sum_{j=1}^k(t\pi_j-\pi_j^{\mathrm{ref}})_+,
\qquad
U(u)
=
\sum_{j=1}^k(\pi_j^{\mathrm{ref}}-u\pi_j)_+.
\]
The increasing order of the ratios gives
\[
L(t)
=
\max\left\{
0,\max_{1\leq i\leq k}(tP_i-Q_i)
\right\},
\]
because the positive terms correspond to an initial
segment of the ordered labels. Similarly,
\[
U(u)
=
\max\left\{
0,\max_{1\leq l\leq k}
\left(Q_l^{\mathrm{tail}}-uP_l^{\mathrm{tail}}\right)
\right\},
\]
because its positive terms correspond to a final
segment of the ordered labels.

For $0<t\leq1$, the condition $L(t)\leq\xi$ is
equivalent to $t\leq t_\xi$. For $u\geq1$,
the condition $U(u)\leq\xi$ is equivalent to
$u\geq u_\xi$.

For a finite radius $\rho$, put $r=e^{-\rho}$.
Using the bounded simplices defined in
Lemma~\ref{lem:box-projection}, we have
\[
\mathcal K_\rho(\pi)
=
\bigcup_{r\leq t\leq1}\mathcal B_{t,t/r}.
\]
Indeed, normalising a vector with coordinates
between $r\pi_j$ and $\pi_j$ produces bounds
$t\pi_j$ and $(t/r)\pi_j$ for some $t\in[r,1]$.
Conversely, multiplying a vector in
$\mathcal B_{t,t/r}$ by $r/t$ gives a vector in
the original box.

By Lemma~\ref{lem:box-projection}, a bounded simplex
$\mathcal B_{t,u}$ has distance at most $\xi$ from
$\pi^{\mathrm{ref}}$ exactly when
\[
0<t\leq t_\xi,\qquad u\geq u_\xi.
\]
When $t_\xi>0$, the smallest possible ratio $u/t$
is $u_\xi/t_\xi$. Since $u/t=e^\rho$, this proves
\eqref{eq:inverse-radius}.
If $t_\xi=0$, no positive lower scale is feasible,
so no finite radius attains the tolerance.

To obtain the distance formula, first consider
$0\leq\xi<\operatorname{TV}(\pi^{\mathrm{ref}},\pi)$.
Since
\[
L(1)=U(1)
=\operatorname{TV}(\pi^{\mathrm{ref}},\pi),
\]
we have $t_\xi<1$ and $u_\xi>1$.
The condition $u_\xi\leq t_\xi/r$ is therefore
equivalent to
\[
\frac{Q_l^{\mathrm{tail}}-\xi}
     {P_l^{\mathrm{tail}}}
\leq
\frac{\xi+Q_i}{rP_i}
\qquad
\text{for every }i,l.
\]
Rearranging gives
\[
\xi
\geq
\frac{
rP_iQ_l^{\mathrm{tail}}-P_l^{\mathrm{tail}}Q_i
}{
rP_i+P_l^{\mathrm{tail}}
}
\qquad
\text{for every }i,l.
\]

Write
$V=\operatorname{TV}(\pi^{\mathrm{ref}},\pi)$
and let $M$ denote the right-hand side of
\eqref{eq:tv-distance}.
The preceding equivalence shows that, for
$0\leq\xi<V$,
\[
\delta_\rho(\pi^{\mathrm{ref}},\pi)\leq\xi
\quad\Longleftrightarrow\quad
M\leq\xi.
\]
Both quantities belong to $[0,V]$.
For the distance, this follows because
$\pi\in\mathcal K_\rho(\pi)$.
For $M$, the inequalities
\[
P_i-Q_i\leq V,
\qquad
Q_l^{\mathrm{tail}}-P_l^{\mathrm{tail}}\leq V
\]
give
\[
rP_i\bigl(Q_l^{\mathrm{tail}}-V\bigr)
\leq rP_iP_l^{\mathrm{tail}}
\leq P_l^{\mathrm{tail}}(Q_i+V),
\]
and hence every term defining $M$ is at most $V$.
The equivalence therefore implies
$\delta_\rho(\pi^{\mathrm{ref}},\pi)=M$.
When $V=0$, both quantities are zero directly.
This proves \eqref{eq:tv-distance} for every
finite radius, including $\rho=0$.

For $\xi\geq V$, the centre $\pi$ already attains
the tolerance at radius zero, and
$t_\xi=u_\xi=1$, consistently with
\eqref{eq:inverse-radius}.
At $\rho=\infty$, the set is the full simplex
and the distance is zero.

Sorting takes $O(k\log k)$ operations, and computing
the cumulative sums takes $O(k)$.
The inverse radius then requires $O(k)$ operations,
while the double maximum in the distance formula
requires $O(k^2)$.
\end{proof}

By \eqref{eq:hilbert}, the inverse-radius problem
searches for the smallest projective distance from
the centre to a probability vector within total
variation distance $\xi$ of the reference vector.
Related smoothing of Hilbert's projective metric
using a trace-distance tolerance is considered by
\citet[Appendix~A.7]{shi2026smoothing}.
For diagonal probability matrices, trace distance
equals the total variation distance between their
diagonal entries.

\paragraph{Finite Reference Hulls.}
Distance to a nonempty closed convex set is convex
in its argument. To see this, choose a closest point
for each of two probability vectors, take the same
convex combination of the closest points, and use
convexity of the norm defining total variation.
Consequently, for any nonnegative weights
$w_1,\ldots,w_J$ summing to one,
\[
\delta_\rho\left(
\sum_{s=1}^Jw_s\pi_s^{\mathrm{ref}},\pi
\right)
\leq
\sum_{s=1}^Jw_s
\delta_\rho(\pi_s^{\mathrm{ref}},\pi)
\leq
\max_{1\leq s\leq J}
\delta_\rho(\pi_s^{\mathrm{ref}},\pi).
\]
Each supplied reference vector also belongs to its
convex hull. This proves
\eqref{eq:reference-distance}.
A radius satisfies $D_\rho(x)\leq\xi$ exactly when
it satisfies the distance requirement for every
reference vector. Hence the smallest such radius is
\[
\max_{1\leq s\leq J}
s_\xi(\pi_s^{\mathrm{ref}}(x);\pi(x)).
\]

\begin{proof}[Proof of Proposition~\ref{prop:risk-duality}]
Fix $x$ and $\rho$, and write
\[
\mathcal A=\mathcal P_{\mathrm{ref}}(x),
\qquad
\mathcal B=\mathcal K_\rho(\pi(x)).
\]
These sets are nonempty, compact and convex.
For probability vectors,
\[
\operatorname{TV}(\pi^{\mathrm{ref}},\tilde{\pi})
=
\max_{f\in[0,1]^k}
(\pi^{\mathrm{ref}}-\tilde{\pi})^\top f.
\]
For a fixed $\pi^{\mathrm{ref}}$, the objective is
continuous and bilinear in $\tilde{\pi}$ and $f$.
Convex minimax duality therefore gives
\[
\begin{aligned}
\min_{\tilde{\pi}\in\mathcal B}
\operatorname{TV}(\pi^{\mathrm{ref}},\tilde{\pi})
&=
\max_{f\in[0,1]^k}
\min_{\tilde{\pi}\in\mathcal B}
(\pi^{\mathrm{ref}}-\tilde{\pi})^\top f
\\
&=
\max_{f\in[0,1]^k}
\left\{
(\pi^{\mathrm{ref}})^\top f-h_{\mathcal B}(f)
\right\}.
\end{aligned}
\]
Maximising over $\pi^{\mathrm{ref}}\in\mathcal A$
and interchanging the two maxima gives
\[
D_\rho(x)
=
\max_{f\in[0,1]^k}
\left\{h_{\mathcal A}(f)-h_{\mathcal B}(f)\right\}.
\]
Compactness ensures that the extrema are attained.
The maximum is nonnegative because $f=0$ is
admissible.
\end{proof}

\subsection{Calibration and its Qualifications}

\begin{proof}[Proof of Theorem~\ref{thm:radius-crc}]
Condition on $\mathcal F_{\mathrm{fit}}$ and write
$L_i(\rho)=D_\rho(X_i)$.
The functions $L_i$ are continuous at finite radii,
nonincreasing, bounded in $[0,1]$, and tend to zero
as $\rho\to\infty$. Conditional exchangeability
of the inputs makes these random functions
exchangeable.

For $\zeta>1/(m+1)$, define the auxiliary radius
\[
\rho^\circ
=
\inf\left\{
\rho\geq0:
\frac1{m+1}\sum_{i=1}^{m+1}L_i(\rho)
\leq\zeta
\right\}.
\]
Both this threshold and the corrected empirical
threshold in \eqref{eq:radius-crc} are finite and
attained, by continuity and their limiting values.

Since $L_{m+1}(\rho)\leq1$, every radius satisfying
the corrected empirical criterion also satisfies
the criterion defining $\rho^\circ$. Thus
\[
\rho^\circ\leq\widehat\rho_\zeta,
\qquad
L_{m+1}(\widehat\rho_\zeta)
\leq L_{m+1}(\rho^\circ).
\]
The radius $\rho^\circ$ is unchanged by permuting
the $m+1$ loss functions. Conditional exchangeability
therefore gives
\[
\begin{aligned}
\mathbb E\left[
L_{m+1}(\rho^\circ)
\mid\mathcal F_{\mathrm{fit}}
\right]
&=
\mathbb E\left[
\frac1{m+1}\sum_{i=1}^{m+1}L_i(\rho^\circ)
\mid\mathcal F_{\mathrm{fit}}
\right]
\\
&\leq\zeta.
\end{aligned}
\]
Combining these inequalities proves
\eqref{eq:reference-guarantee}.
The definition and attainment of the corrected
threshold establish its minimality.

For $0\leq\zeta\leq1/(m+1)$, the full-simplex
fallback has zero discrepancy at every input.
The result then holds directly.
This is conformal risk control
\citep{angelopoulos2024crc} applied to the reference
approximation error.
\end{proof}

\begin{proof}[Proof of Corollary~\ref{cor:reference-transfer}]
Fix an input $x$ and a radius $\rho$.
For each
$\pi^{\mathrm{ref}}\in\mathcal P_{\mathrm{ref}}(x)$,
compactness supplies
$\tilde{\pi}\in\mathcal K_\rho(\pi(x))$ satisfying
\[
\operatorname{TV}(\pi^{\mathrm{ref}},\tilde{\pi})
\leq D_\rho(x).
\]
Expectations of a loss taking values in $[0,1]$
differ by at most this total variation distance.
Maximising over the reference set proves the
decision-risk inequality.

For any event $R\subseteq\mathcal Y$,
\[
\sum_{j\in R}\pi_j^{\mathrm{ref}}
\geq
\sum_{j\in R}\tilde{\pi}_j-D_\rho(x)
\geq
\underline P_{\rho,x}(R)-D_\rho(x).
\]
Taking the minimum over the reference set proves
the lower-probability inequality.
It holds for every event and radius, so it also
holds for the selected radius
$\widehat\rho_\zeta$ and the data-dependent region
$\widehat R(x)$.

Apply the assumed IDEC lower-probability constraint
and take expectations conditional on
$\mathcal F_{\mathrm{fit}}$.
Theorem~\ref{thm:radius-crc} then gives
\[
\mathbb E\left[
\underline P_{\mathrm{ref},X_{m+1}}
\bigl(\widehat R(X_{m+1})\bigr)
\mid\mathcal F_{\mathrm{fit}}
\right]
\geq1-\gamma-\zeta.
\]

For true label coverage, compactness of the reference
set supplies a reference vector at distance
$b_{\mathrm{ref}}(x)$ from $\pi^\star(x)$.
Consequently,
\[
\sum_{j\in\widehat R(x)}\pi_j^\star(x)
\geq
\underline P_{\mathrm{ref},x}
\bigl(\widehat R(x)\bigr)-b_{\mathrm{ref}}(x).
\]
Under the stated conditional label law, the
left-hand side is the conditional probability
that the new label belongs to the selected region,
given the calibration inputs, the new input and
$\mathcal F_{\mathrm{fit}}$.
Taking expectations and applying the preceding
reference bound proves the marginal coverage result.
\end{proof}

\paragraph{One Sided Uncertainty Implications.}
For $G(\pi)=1-\sum_j\pi_j^2$,
\[
|G(\pi)-G(\tilde{\pi})|
\leq2\operatorname{TV}(\pi,\tilde{\pi}),
\]
as established in Appendix~\ref{app:stability}.
For a fixed input $x$, the directed approximation
error therefore gives
\[
\begin{aligned}
\mathrm{AU}_G(\mathcal K_\rho(\pi(x)))
&\leq
\mathrm{AU}_G(\mathcal P_{\mathrm{ref}}(x))
+2D_\rho(x),
\\
\mathrm{TU}_G(\mathcal K_\rho(\pi(x)))
&\geq
\mathrm{TU}_G(\mathcal P_{\mathrm{ref}}(x))
-2D_\rho(x),
\\
\mathrm{EU}_G(\mathcal K_\rho(\pi(x)))
&\geq
\mathrm{EU}_G(\mathcal P_{\mathrm{ref}}(x))
-4D_\rho(x).
\end{aligned}
\]
These inequalities allow IDEC to have a larger
uncertainty gap than the reference.
For example, the full simplex has zero directed
discrepancy from every reference set.
Controlling both directed distances also bounds
discrepancies between the corresponding lower
and upper uncertainty terms.

\paragraph{Population Width Consistency Under a Strict Crossing.}
Condition on $\mathcal F_{\mathrm{fit}}$, keep the
fitted predictors fixed, and suppose the calibration
inputs form an i.i.d. sequence from the target
population. Define
\[
R_{\mathrm{app}}(\rho)
=
\mathbb E\left[
D_\rho(X)\mid\mathcal F_{\mathrm{fit}}
\right],
\qquad
\rho^\star_\zeta
=
\inf\{\rho\geq0:R_{\mathrm{app}}(\rho)\leq\zeta\},
\]
where $\zeta\in(0,1)$ is fixed.
Suppose $0<\rho^\star_\zeta<\infty$ and, for every
sufficiently small $a>0$,
\[
R_{\mathrm{app}}(\rho^\star_\zeta-a)
>
\zeta
>
R_{\mathrm{app}}(\rho^\star_\zeta+a).
\]
Then the radius selected by
\eqref{eq:radius-crc} converges almost surely to
$\rho^\star_\zeta$ as the calibration sample
size tends to infinity.

Indeed, apply the strong law at the two fixed
radii $\rho^\star_\zeta-a$ and
$\rho^\star_\zeta+a$.
The empirical averages converge to their population
values, and the conformal correction vanishes.
Eventually, the lower radius fails the corrected
criterion and the upper radius satisfies it.
Monotonicity therefore places the selected radius
between them.
Applying this argument along a countable sequence
$a\downarrow0$ proves convergence.

The conclusion concerns the specified reference
approximation objective and fixed fitted predictors.
Flat risk plateaus and predictors changing with
calibration sample size require separate assumptions.
\subsection{Mean Shift Induced by Clipped Gaussian Corruption}
\label{app:clipped-gaussian-bias}
For $Z\sim N(0,1)$, $x\in[0,1]$ and $\sigma>0$, define
$C_\sigma(x,Z)=\min\{1,\max\{0,x+\sigma Z\}\}$.
Let $\phi$ and $\Phi$ be the standard normal density and distribution functions.
Writing $a=-x/\sigma$ and $b=(1-x)/\sigma$, direct integration gives
\[
\mathbb E C_\sigma(x,Z)
=x\{\Phi(b)-\Phi(a)\}
+\sigma\{\phi(a)-\phi(b)\}+1-\Phi(b).
\]
In particular, a zero-valued background pixel has positive expected intensity
\[
\mathbb E C_\sigma(0,Z)
=\sigma\{\phi(0)-\phi(1/\sigma)\}+1-\Phi(1/\sigma).
\]
At $\sigma=0.2$ this is approximately $0.0798$. Thus clipped Gaussian corruption
combines random perturbation with an input-dependent mean shift. With many
zero-valued background pixels, this shift can affect encoder activations
systematically. The formula alone does not establish the magnitude of a
classifier's degradation; that requires the controlled model comparisons.

\begin{proof}
On $a<Z<b$, the corrupted value is $x+\sigma Z$. On $Z\geq b$ it is one,
and on $Z\leq a$ it is zero. Integrating over these events, using
$\int_a^b z\phi(z)\,dz=\phi(a)-\phi(b)$, gives the result.
\end{proof}

\section{Conditional Stability of Predictive Credal Inference}\label{app:stability}

Let $\mathcal B,\mathcal C$ be nonempty compact predictive credal sets on
$k\geq2$ labels. Their two sided Hausdorff distance is
\[
d_H^{\operatorname{TV}}(\mathcal B,\mathcal C)
=\max\left\{\sup_{b\in\mathcal B}\inf_{c\in\mathcal C}\operatorname{TV}(b,c),
\sup_{c\in\mathcal C}\inf_{b\in\mathcal B}\operatorname{TV}(b,c)\right\}.
\]
For $0\leq\delta\leq1$, put $\bar\delta=\min\{\delta,1-1/k\}$ and
$\omega_k(\delta)=h_2(\bar\delta)+\bar\delta\log_2(k-1)$.
This is a monotone envelope of the finite alphabet entropy continuity bound
\citep{audenaert2007sharp}.

\begin{proposition}\label{prop:stability}
If $d_H^{\operatorname{TV}}(\mathcal B,\mathcal C)\leq\delta$, then, for every
label event $A$,
\[
|\underline P_{\mathcal B}(A)-\underline P_{\mathcal C}(A)|\leq\delta,
\qquad
|\overline P_{\mathcal B}(A)-\overline P_{\mathcal C}(A)|\leq\delta.
\]
Applying the same Shannon functional to both sets gives AU and TU discrepancies
at most $\omega_k(\delta)$, and EU discrepancy at most $2\omega_k(\delta)$.
Applying the same Gini functional gives the respective bounds $2\delta$,
$2\delta$ and $4\delta$.
For $0\leq\ell(a,j)\leq1$ let
$R_{\mathcal K}(a)=\sup_{p\in\mathcal K}\sum_jp_j\ell(a,j)$.
Then $|R_{\mathcal B}(a)-R_{\mathcal C}(a)|\leq\delta$. If minimisers exist
and $a_{\mathcal C}$ minimises $R_{\mathcal C}$, then
\[
R_{\mathcal B}(a_{\mathcal C})-\min_aR_{\mathcal B}(a)\leq2\delta.
\]
\end{proposition}

When two predictive credal sets are close in the stated
Hausdorff distance, their probability bounds and uncertainty
scores are correspondingly close. A decision that minimises
worst case expected loss under $\mathcal C$ also performs
nearly optimally under $\mathcal B$: its worst case expected
loss exceeds the minimum achievable under $\mathcal B$
by at most $2\delta$.

\begin{proof}[Proof of Proposition \ref{prop:stability}]
Paired distributions within total variation $\delta$ differ by at most $\delta$
on every event and every expectation of a $[0,1]$-valued function.
Apply this fact in both directions and take infima or suprema.
The entropy-continuity bound applies in the same way to the extrema of Shannon
entropy. If $G(p)=1-\sum_jp_j^2$, write
$G(p)-G(q)=-\sum_j(p_j-q_j)(p_j+q_j)$.
The positive and negative parts of $p-q$ have equal mass
$\operatorname{TV}(p,q)$; the range of $p+q$ is at most two. Hence
$|G(p)-G(q)|\leq2\operatorname{TV}(p,q)$, proving the Gini bounds.
Subtracting the two entropy extrema gives the corresponding EU errors.
Finally, for a minimiser $a_{\mathcal B}$,
\[
R_{\mathcal B}(a_{\mathcal C})
\leq R_{\mathcal C}(a_{\mathcal C})+\delta
\leq R_{\mathcal C}(a_{\mathcal B})+\delta
\leq R_{\mathcal B}(a_{\mathcal B})+2\delta.
\]
\end{proof}

\begin{corollary}\label{cor:two-sided-region}
Under the assumptions of Proposition~\ref{prop:stability}, if $\underline P_{\mathcal C}(R)\geq1-\gamma$, then
$\underline P_{\mathcal B}(R)\geq1-\gamma-\delta$.
Let $m_{\mathcal K}(t)$ be the smallest cardinality of an event having lower
probability at least $t$. If $0\leq t-\delta\leq t+\delta\leq1$, then
\[
m_{\mathcal B}(t-\delta)\leq m_{\mathcal C}(t)\leq m_{\mathcal B}(t+\delta).
\]
\end{corollary}

A prediction region constructed using $\mathcal C$ therefore
retains its probability guarantee under $\mathcal B$ with
a reduction of at most $\delta$. The cardinality bounds
compare the minimum region size under $\mathcal C$ at
level $t$ with those under $\mathcal B$ at levels
$t-\delta$ and $t+\delta$. If the latter two sizes coincide,
the minimum region size under $\mathcal C$ is the same.

\begin{proof}[Proof of Corollary \ref{cor:two-sided-region}]
Transfer the lower-probability constraint in the appropriate direction using
Proposition~\ref{prop:stability}. Every feasible event on one side is feasible
at the shifted level on the other, giving both cardinality inequalities.
\end{proof}

For finite hulls $\mathcal B=\operatorname{Conv}\{b_s\}$ and
$\mathcal C=\operatorname{Conv}\{c_t\}$,
\[
d_H^{\operatorname{TV}}(\mathcal B,\mathcal C)
=\max\left\{\max_s\min_{w\in\Delta}\tfrac12\|b_s-\sum_tw_tc_t\|_1,
\max_t\min_{v\in\Delta}\tfrac12\|c_t-\sum_sv_sb_s\|_1\right\}.
\]
It suffices to compute the distance from each supplied
predictive vector to the other credal set and take the
largest value across both collections. Each minimisation
can be solved as a linear programme.

In the MNIST comparison of Section~\ref{sec:experiments},
CBDL-A3 predictive vectors are obtained by averaging
predictions over samples from fitted variational
posteriors. Applying the distance above to these
vectors measures agreement with that numerical
reference. 

\section{Relation to the Imprecise Dirichlet Model and
the Naive Credal Classifier}\label{app:idm-ncc}

The Imprecise Dirichlet Model (IDM), the Naive Credal
Classifier (NCC), and CDEC use sets of probability
distributions to represent uncertainty and support
prediction. Their relationships can be understood by
examining where predictive imprecision enters and how
the resulting predictions are used.

\paragraph{Prior Variation in the Imprecise Dirichlet Model.}
The IDM extends conjugate Bayesian analysis by considering
a family of Dirichlet priors
\citep{walley_marbles,bernard_bom}. For a fixed prior
strength $\eta>0$, this family consists of
\[
Q_{0,t}=\operatorname{Dir}(\eta t),
\qquad
t\in\Delta^{k-1},\quad t_j>0\ \text{for every }j.
\]
The vector $t$ is the prior mean and $\eta$ is the sum
of the Dirichlet parameters. Suppose that a Categorical
sample has class counts $n_1,\ldots,n_k$, with
$N=\sum_j n_j$. Updating each prior gives the posterior
$\operatorname{Dir}(n+\eta t)$ and predictive probabilities
\[
p_j(t)=\frac{n_j+\eta t_j}{N+\eta}.
\]
Their lower and upper envelopes are
\[
\underline p_j=\frac{n_j}{N+\eta},
\qquad
\overline p_j=\frac{n_j+\eta}{N+\eta},
\]
respectively. The endpoints are limiting values as the prior mean
approaches the boundary of the simplex. Equivalently,
they are attained in the closure of the predictive set.

This construction expresses sensitivity to the prior
mean while fixing the prior strength and the sampling
model. Its predictive interval width is
\[
\overline p_j-\underline p_j=\frac{\eta}{N+\eta}.
\]
For fixed $\eta$, the width decreases with the sample
size and is the same for all observed class proportions; it therefore quantifies the effect of varying the prior
mean.

\paragraph{Evidence variation in CDEC.}
CDEC fixes the common evidential prior
$\operatorname{Dir}(\mathbf1_k)$ and retains several
learned evidence factors at the input $\tilde x$.
Model $s$ supplies virtual counts $c_s(\tilde x)$,
giving predictive probabilities
\[
p_{s,j}(\tilde x)
=
\frac{1+c_{s,j}(\tilde x)}
{k+\sum_{l=1}^k c_{s,l}(\tilde x)}.
\]
Both the relative virtual counts and their total can
vary across models and inputs. The predictive credal
set therefore represents variation across the retained
evidential specifications. This provides a predictive
form of sensitivity analysis under a common prior,
as formalised by Proposition~\ref{prop:affine}.

The IDM and CDEC both use the Dirichlet-Categorical
conjugate structure. The likelihood of an independent
Categorical sample and the Multinomial likelihood of
its class counts differ only by a factor independent
of the class probability vector. They consequently
give the same Dirichlet update. The substantive
distinction concerns the prior or evidence factors
allowed to vary.

The prior $\operatorname{Dir}(\mathbf1_k)$ supplies
a symmetric baseline: its mean is uniform and it
contributes one virtual count to each class, with
total concentration $k$. CDEC therefore examines how predictions and decisions depend
on the chosen evidence factors. A prediction region returned by CDEC
has at least the required probability under any additional
model whose predictive distribution belongs to the credal set; if that distribution is within total variation distance
$\delta$ of the set, the region's probability under that
model is at least $1-\gamma-\delta$.

\paragraph{Predictive Modelling and Classification.}
The NCC extends naive Bayes classification by allowing
sets of class and class conditional probabilities
\citep{zaffalon}. It forms part of the broader
framework of credal networks \citep{cozman3}, and
the IDM can be used to construct its probability
sets from observed counts \citep[Chapter 10]{augustin}.
The naive Bayes structure assumes that the input
attributes are conditionally independent given the
class. CDEC obtains its predictions from learned
representations, allowing dependencies among the
original input attributes to be represented by
the evidential models.

The NCC uses dominance comparisons to identify
classes that remain plausible decisions
\citep{zaffalon}; these comparisons can return
several labels when the predictive alternatives
do not support a unique choice. CDEC's IHDR
criterion selects a smallest collection of labels
whose probability reaches the prescribed level
under every retained distribution.

These decision criteria can produce different
regions even when the predictive distribution is
precise. For example, with class probabilities
$(0.8,0.2)$, probability dominance selects the
first class, while a $95\%$ IHDR contains both
classes. The former identifies the preferred
class; the latter satisfies a specified probability
requirement for the returned region. When one class
is required, \eqref{eq:cdec-point} supplies CDEC's
decision by minimising the worst case error
probability among the individual class choices.

\section{Optional High Probability Reference-Risk Control}
\label{app:hp-calibration}

We give a calibration rule that bounds the calibrated
IDEC predictor's expected approximation error on new
inputs, with high probability over the calibration
sample.

Throughout, we condition on $\mathcal F_{\mathrm{fit}}$
from Section~\ref{sec:radius-calibration}, treating the
fitted predictors and their tuning choices as fixed.
Let $\mathcal G\subset[0,\infty)$ be a nonempty finite
set of candidate radii specified before calibration.
Under this conditioning, assume that the calibration
inputs $X_1,\ldots,X_m$, with $m\geq2$, and the new
input $X_{\mathrm{new}}$ are independent draws from
the same input distribution.

For each $\rho\in\mathcal G$, define
\[
 \bar D_\rho
=
\frac1m\sum_{i=1}^mD_\rho(X_i),
\qquad
\hat V_\rho
=
\frac1{m-1}\sum_{i=1}^m
\left(D_\rho(X_i)- \bar D_\rho\right)^2.
\]
Fix a failure probability $\delta\in(0,1)$.
We apply the empirical Bernstein inequality
\citep[Theorem~4]{maurer2009empirical} to the independent
losses $D_\rho(X_i)\in[0,1]$. For each candidate radius,
this bounds the population approximation error using
the sample mean and variance. Assigning failure probability
$\delta/|\mathcal G|$ to each radius and applying a union
bound makes these bounds hold simultaneously with
probability at least $1-\delta$ over the calibration
sample. The resulting upper confidence bound is
\begin{equation}\label{eq:hp-upper}
U_\rho
=
 \bar D_\rho
+
\sqrt{
\frac{
2\hat V_\rho
\log(2\lvert\mathcal G\rvert/\delta)
}{m}
}
+
\frac{
7\log(2\lvert\mathcal G\rvert/\delta)
}{
3(m-1)
}.
\end{equation}

Given a budget $\zeta\in[0,1)$, choose the smallest
$\rho\in\mathcal G$ satisfying $U_\rho\leq\zeta$.
Denote the selected radius by $\hat\rho$; if no grid radius satisfies the condition, set
$\hat\rho=\infty$.
For a fixed radius, define the population
approximation error by
\[
R_{\mathrm{app}}(\rho)
=
\mathbb E\left[
D_\rho(X_{\mathrm{new}})
\mid\mathcal F_{\mathrm{fit}}
\right].
\]
For the selected radius,
$R_{\mathrm{app}}(\hat\rho)$ equals the expected
error on the independent new input conditional on
both $\mathcal F_{\mathrm{fit}}$ and the realised
calibration sample.

\begin{proposition}[Finite Grid Population Risk Control]
\label{prop:hp-grid}
Suppose that the fitted predictors, their tuning choices
and a nonempty finite set of candidate radii
$\mathcal G\subset[0,\infty)$ are fixed before calibration.
Conditional on $\mathcal F_{\mathrm{fit}}$, assume that
$X_1,\ldots,X_m$, with $m\geq2$, and $X_{\mathrm{new}}$
are independent draws from the same input distribution.
Fix $\delta\in(0,1)$ and $\zeta\in[0,1)$, and compute
$U_\rho$ using \eqref{eq:hp-upper}. Let $\hat\rho$
be the smallest radius in $\mathcal G$ satisfying
$U_\rho\leq\zeta$, with $\hat\rho=\infty$ if none exists.
Then,
\[
\Pr\nolimits_{\mathrm{cal}}\!\left(
R_{\mathrm{app}}(\hat\rho)\leq\zeta
\mid\mathcal F_{\mathrm{fit}}
\right)
\geq1-\delta.
\]
Here, $\Pr_{\mathrm{cal}}$ denotes probability over the
calibration inputs $X_1,\ldots,X_m$, which determine
$\hat\rho$. The quantity $R_{\mathrm{app}}(\hat\rho)$
is the expected approximation error on a new input
after that radius has been selected.
\end{proposition}

Thus, with confidence at least $1-\delta$, the selected
IDEC predictor has average approximation error at most
$\zeta$ on new inputs from the same population.
The guarantee accounts for choosing the radius using
the calibration data, because the confidence bounds
hold simultaneously for every candidate radius.

\begin{proof}[Proof of Proposition \ref{prop:hp-grid}]
For each fixed $\rho\in\mathcal G$, apply the
empirical Bernstein inequality
\citep[Theorem~4]{maurer2009empirical} to the
independent losses $D_\rho(X_i)\in[0,1]$,
with failure probability
$\delta/\lvert\mathcal G\rvert$.
The variance in that theorem is the unbiased
sample variance $\hat V_\rho$ used here.
It follows that
$R_{\mathrm{app}}(\rho)\leq U_\rho$
except on an event of probability at most
$\delta/\lvert\mathcal G\rvert$.

A union bound makes these inequalities simultaneous
over all radii in $\mathcal G$, with probability
at least $1-\delta$.
On this event, every selected grid radius satisfies
\[
R_{\mathrm{app}}(\hat\rho)
\leq U_{\hat\rho}\leq\zeta.
\]
If no grid radius qualifies, the full-simplex
fallback has population approximation error zero.
\end{proof}

In the MNIST experiments, we use $\delta=0.05$ and
a grid consisting of zero and $128$ logarithmically
spaced radii from $0.01$ to $20$. These settings
were fixed before calibration and test evaluation.
\section{Independent MNIST Protocol and Baseline Fidelity}\label{app:mnist-protocol}

\paragraph{Data Chronology.}
The NumPy archive contains the standard MNIST training
set of $60000$ images and test set of $10000$ images. Its SHA256 matches the Keras-distributed archive.
Seed $20261116$ generates an unconditional random permutation of the training
indices, allocating $50000$, $5000$, and $5000$ cases to fitting, tuning and
calibration. The three runs share the same data partitions; training uses only the original images, with pixel
values divided by $255$. All model checkpoints
and calibration settings are fixed before calibration.
The calibrated thresholds are then fixed before
evaluating the test set.

\paragraph{Architectures and Fitting.}
The three seed blocks are $(322,365,382)$, $(410,411,412)$ and $(413,414,415)$.
The corresponding members have LeakyReLU slope $0.1$, tanh and SiLU hidden
activations. Their encoders have dimensions $784$--$64$--$64$--$6$.
Bayesian networks include a categorical linear output and independent Gaussian
variational factors for all weights and biases. CDEC uses latent batch
normalisation and six radial density layers per class; the Bayesian likelihood
has no batch-dependent normalisation. 

Adam uses batch size $256$ and one CPU thread per fit. Bayesian training includes
a charged ten-epoch MAP initialisation followed by 100 VI epochs. The objective
is mean expected categorical negative log likelihood, plus the full sum of Gaussian
KL terms divided by the fitting size \citep{blundell2015weight}. Two local
reparameterisation draws estimate each minibatch expectation. 
A predefined late improvement trigger extends the
run by 100 epochs at reduced learning rate. Predictive validation is evaluated
at epoch one and every five epochs. Configurations use a mean tuning NLL across
the three first block members, with a 100 draw check. We compare two priors with independent Gaussian
weights and biases, all with mean zero. The first
assigns standard deviation $2/\sqrt{h}$ to weights
in a layer with $h$ inputs and standard deviation
$2$ to biases. The second assigns standard deviation
$0.5$ to every weight and bias. Each candidate prior
is evaluated across the three Bayesian networks.
The tuning criterion selects the second prior,
which is used for all three networks. The learning rate is $0.001$.

Evidential fitting uses UCE with a Dirichlet entropy
regularisation coefficient of $10^{-6}$. We compare
learning rates $0.001$ and $0.003$. Each initial
training run lasts between $30$ and $50$ epochs;
the prespecified continuation rule can extend the
maximum to $100$ epochs. We retain the checkpoint with the lowest predictive
negative log likelihood on the tuning data. The
selected learning rate is $0.001$. The models selected
during the learning rate comparison also supply the
first of the three reported runs. Each fit is counted
once in the total training cost.
All rejected models, histories and costs are retained.

\paragraph{Frozen Monte Carlo References.}
Each Bayesian model receives a fixed bank of 100 global weight draws with seed
equal to its model seed plus $910000$. The primary $M=20$ and sensitivities
$M\in\{5,100\}$ use per-layer prefixes of that bank. The same draws are used for every
input and across calibration and test. Each Bayesian checkpoint selected using the variational
training objective is evaluated with a separate,
fixed set of posterior weight samples, reused across
all inputs. Weight sample probabilities are averaged within each posterior
predictive before the three vectors are convexified. The Bayesian benchmark constructs its predictive
credal set as the convex hull of three probability
vectors, each computed by averaging predictions over
a fixed set of weight samples from the corresponding
fitted variational posterior.

\paragraph{CD-CI.}
Recall that the predictive class probabilities from
PostNet $s\in\{1,2,3\}$ are
\[
\pi_{s,j}(x)
=
\frac{\alpha_{s,j}(x)}
     {\sum_{l=1}^{k}\alpha_{s,l}(x)},
\qquad j\in\{1,\ldots,k\},
\]
where $\alpha_s(x)$ is the Dirichlet parameter vector
produced by that network. For CD-CI
\citep{huang2025distilling}, we use the average of
the three CDEC predictive vectors as the reference
and the first PostNet's predictive vector as the student,
\[
\pi^{\mathrm{ref}}(x)
=
\frac13\sum_{s=1}^{3}\pi_s(x),
\qquad
\pi(x)=\pi_1(x).
\]
The principal comparison uses the $\alpha$-divergence
\[
D_\alpha\bigl(\pi^{\mathrm{ref}}(x),\pi(x)\bigr)
=
\frac{
\sum_{j=1}^{k}
\bigl(\pi_j^{\mathrm{ref}}(x)\bigr)^\alpha
\bigl(\pi_j(x)\bigr)^{1-\alpha}
-1
}{
\alpha(\alpha-1)
}.
\]
We set $\alpha=0.9$, following the public example
implementation. TV and forward KL divergence
$\mathrm{KL}\bigl(\pi^{\mathrm{ref}}(x)\Vert\pi(x)\bigr)$
provide additional controls.

Each divergence is calibrated separately using its
values between $\pi^{\mathrm{ref}}(X_i)$ and $\pi(X_i)$
on the $m$ calibration inputs. To target $90\%$ containment, we choose the threshold
as the $\lceil0.9(m+1)\rceil$th smallest calibration
score. Under exchangeability, the $r$th calibration
order statistic gives containment probability at
least $r/(m+1)$, so this is the smallest rank meeting
the target. This probability averages over the
calibration sample and the new input.
At a new input, the resulting credal set contains
all probability vectors whose divergence from
$\pi(x)$ is at most this threshold. If the required
rank exceeds $m$, we use the full simplex.
The calibration target is a $90\%$ probability that
the set contains the reference vector
$\pi^{\mathrm{ref}}(x)$ at a new input.

We compute the lower and upper probability of each
class, and the lower probability of each label set,
over the CD-CI predictive credal set defined by
the calibrated divergence threshold.
For a nonempty proper label set $R$, the optimisation
reduces exactly to two outcomes: a label in $R$
or a label in its complement. For any prescribed
probability of $R$, the minimum divergence is attained
by preserving the student's conditional class
probabilities within each group. We solve the resulting
scalar problems numerically. Our CD-CI implementation computes class probability
bounds and prediction regions for all ten MNIST classes.
The image classification code released by
\citet{huang2025distilling} uses a discretised
simplex for three CIFAR-10 classes.

The intersection probability summary interpolates
between the lower and upper class probability bounds
using the same weight for every class, chosen so that
the resulting probabilities sum to one. For the
principal CD-CI comparison, reported timings include
predictive probabilities, prediction regions and this
probability summary.

\paragraph{CreDE and CREDIT.}
CreDE uses the interval softmax, reachable marginal bounds and mean ensemble
aggregation in \citet{wang2024crede}. The loss adds the mean upper cross entropy
to the mean of the largest lower cross entropies on half the minibatch, using
base two logarithms. The fraction is fixed at $0.5$. The common encoder feeds
$2k$ logits. Point decisions use the largest lower probability; proper probability
controls use the explicitly recorded intersection summary. The two learning rate
candidates select $0.003$ by tuning NLL. Independent softmax controls receive
the same two rate search and select $0.001$.

CREDIT implements the teacher wrapper, $2k+1$ output construction, clipped
probability bounds and distillation loss of \citet{wang2026credit}, using the
extended arXiv version's equations (6)-(16). Temperature is fixed at $2.5$.
CREDIT-DE3 uses the matched three member softmax teacher; CREDIT-CDEC3 adapts CREDIT by using the three
CDEC networks as its teacher ensemble. Fitting and tuning
teacher outputs are the only distillation targets. The selected rate is $0.003$,
from $0.001,0.003$, chosen by tuning distillation loss. The maximum is 75 epochs
with a predefined possible 75 epoch continuation. Teacher generation and teacher
training are separate from student fitting costs. Clipping can change realised
interval widths relative to the head's raw width output. 

\paragraph{Calibration and Width Policies.}
We use $\gamma=0.05$ and approximation budgets
$\zeta\in\{0.01,0.02,0.05\}$, with $0.02$ for the
main comparison. IDEC uses the first PostNet with
CDEC and CBDL as separate calibration references.
We cache ratio orderings and cumulative sums.

Expected risk calibration returns $\rho=0$ if it
satisfies \eqref{eq:radius-crc}. Otherwise, an interval
from zero to a radius satisfying the criterion is
halved $48$ times. Each midpoint replaces the upper
endpoint if it satisfies the criterion, and the
lower endpoint otherwise. We return the final upper
endpoint. The upper Gini equation is solved using
$60$ interval halvings.

The high probability rule in Appendix~\ref{app:hp-calibration}
uses a fixed grid of $129$ radii and $95\%$ confidence.
Controls include TV balls calibrated to the same
budgets, $d=1$ and $d=d^\star$. For a TV ball of radius $a$
centred at $\pi$, the distance from a reference vector $q$
is $[\operatorname{TV}(q,\pi)-a]_+$. For every nonempty proper
label set $R$, its lower probability is $[\pi(R)-a]_+$.
Thus, with full support $\pi$ and $a\geq\gamma$, only the
complete label set satisfies the $1-\gamma$ requirement.
Gini scores and candidate regions are recomputed for each width.
Reference predictions are used offline for calibration
and approximation assessment; deployment uses the
first PostNet and the selected width.

The full $5000$ calibration inputs supply the primary thresholds. Twenty fixed
$1000$-case subsamples, seed $85210$, measure conditional calibration variability.
Their complement pool diagnostic was added after finding that the mean approximation error on
the MNIST test set exceeded the calibration budget, and is labelled post test. Neither this analysis nor its outcomes
changes any selected model or radius. Scalar and coordinate-wise $90\%$ Dirichlet credibility controls
are evaluated on a predetermined subset of $1024$ test inputs.
For the scalar control, $512$ draws from the first PostNet's
$\operatorname{Dir}(\alpha(x))$ distribution determine the empirical
$90\%$ quantile of $H_{\mathrm{proj}}(\pi(x),\Pi)$, which is
used as the radius of $\mathcal K_\rho(\pi(x))$.
For the coordinate-wise control, we use the marginal Beta
quantiles at levels $0.1/(2k)$ and $1-0.1/(2k)$, enlarge
the resulting bounds to include $\pi(x)$, and normalise the
corresponding coordinate box. By the union bound, it contains
at least $90\%$ of the probability mass under the Dirichlet
model. On a subset of inputs, independent additional Dirichlet
draws are used to estimate the credible content of the scalar
intervals. These are model-relative credibility checks, not
empirical guarantees for the observed class labels.

\paragraph{Statistics and Secondary Diagnostics.}
For comparisons between methods, we use $1000$ bootstrap
resamples of the test inputs, using the same resampled
inputs for both methods. Each input's contribution to
the reported statistic is first averaged across the
three fitted runs. These intervals describe variation from resampling
the test inputs, with the trained models and
calibration results held fixed. Across fit standard deviations
are reported separately; these blocks share data and are not independent datasets.
 We use AUROC to assess whether each method assigns
higher uncertainty scores to inputs it misclassifies
than to inputs it classifies correctly. For corruption detection, we use AUROC to assess
whether each method assigns higher uncertainty
scores to corrupted images than to clean images. Class-wise metrics,
region lower masses, selective risk and explicit abstention fractions are retained.

For numerical Shannon entropy, the implementation checks the
concavity gap residual and accepts values no larger than
$3\times10^{-6}$ bits. The abstention comparison uses the returned
objective value directly. It does not add the gap to that value or
assign a separate unresolved category near the threshold. The reported
retention rates therefore describe this numerical comparison.
Certifying agreement with the exact entropy decision additionally
requires checking the distance from each value to the threshold
against its error bound.

The PostNet-3 mean and Softmax-3 error ranking controls use the
Shannon ensemble-disagreement score
\[
H\!\left(\frac{1}{3}\sum_{s=1}^{3}\pi_s\right)
-\frac{1}{3}\sum_{s=1}^{3}H(\pi_s).
\]
Their point predictions maximise the corresponding mean probability
vector. For the single PostNet control, the trained network
maps each input $x$ to a Dirichlet parameter vector
$\alpha(x)$. Let $\Pi\sim\operatorname{Dir}(\alpha(x))$
be a random vector of class probabilities. Its uncertainty
score is $H(\mathbb E[\Pi])-\mathbb E[H(\Pi)]$, where
$\alpha_0(x)=\sum_j\alpha_j(x)$ and
\[
\mathbb E[H(\Pi)]
=\frac{\psi(\alpha_0(x)+1)
-\sum_{j=1}^k\pi_j(x)\psi(\alpha_j(x)+1)}{\ln 2}.
\]
These scores differ from the lower-to-upper entropy gap of a
predictive credal set.

The three artificial shifts use the same $2000$ original test identities selected
with seed $739101$: a three pixel right shift, Gaussian noise of standard deviation
$0.2$ in $[0,1]$ units followed by clipping, and a central $6\times6$ occlusion; they are transformations of real images, not independently collected novelty data.
An event-based conservativeness diagnostic maximises over all label subsets the
excess student upper probability over the reference upper probability; it is a
lower bound on reverse directed TV discrepancy.

\paragraph{Computation and Validation.}
We ran the experiments on an AMD EPYC 9V74 processor,
with five CPU cores visible to the execution environment
and no GPU acceleration.
Fits use one PyTorch/BLAS thread; main fits run with at most three workers and
additional baselines with one. Timings are measured on the idle host with three
warmups, nine interleaved repetitions, and batch sizes 1 and 128; model loading
is excluded from warm latency. 
Peak process RSS includes software, data and validation.

Independent checks cover Gaussian KL and gradients, a conjugate posterior,
a scalar logistic posterior against quadrature, radial flow normalisation and
Jacobians, Gini extrema, interval distances against linear programmes, and
CD-CI event/marginal formulas. Some general purpose solvers returned suboptimal
class probability bounds despite reporting convergence.
The reported results use the equations in one variable
derived in the paper, with their numerical solutions
checked independently. A further audit compares \eqref{eq:tv-distance}
with independent linear programmes on $250$ randomly generated
problems, with maximum discrepancy below $2.3\times10^{-16}$;
the inverse radius formula and boundary cases are also checked. The archive audit verifies 1,020,000 repeated region
constraints and 330,000 paired IDEC point evaluations. The largest numerical entropy gap is below $1.68\times10^{-6}$ bits.
All 42 selected models reproduce their cached prediction prefixes in an independent
forward calculation. We repeated the complete training of one evidential
network using the same data, random seed and settings,
and obtained identical trained parameter values.

\section{Classification and Uncertainty under Image Corruption}
\label{app:mnist-corruption}

We examine classification accuracy and uncertainty ranking under three
predeclared transformations of MNIST test images. The transformations
are applied to the same $2000$ test images, selected with seed $739101$:
a three pixel right shift, Gaussian noise with standard deviation
$0.2$ followed by clipping to $[0,1]$, and a central $6\times6$
occlusion. Pixel intensities are measured on the $[0,1]$ scale.
The fitted models and calibrated interval widths are those used in
the clean data evaluation; the models were trained without corruption
augmentation.

For each transformation, we evaluate classification accuracy on the
corrupted images. We also measure how well the epistemic uncertainty
score distinguishes these images from their unmodified counterparts.
Corrupted images are the positive class in this AUROC calculation,
with larger uncertainty scores indicating corruption. The fitting
and evaluation procedures are described in
Appendix~\ref{app:mnist-protocol}.

Table~\ref{tab:mnist-corruption} shows that CDEC is more accurate than
CBDL-A3 under the right shift and central occlusion. Under clipped
Gaussian noise, CDEC's accuracy is $47.92\%$, compared with $90.30\%$
for CBDL-A3. Its epistemic uncertainty score nevertheless distinguishes
corrupted from unmodified inputs effectively, with AUROC $0.934$.
These results demonstrate that effective corruption detection can
coexist with poor classification accuracy on the corrupted inputs.

\begin{table}[h!]
\centering
\small
\begin{tabular}{llrr}
\toprule
Corruption & Method & Accuracy (\%) & EU AUROC \\
\midrule
Right shift & CBDL-A3 & 59.73 & 0.875 \\
Right shift & CDEC-A3 & 61.73 & 0.925 \\
Right shift & IDEC-CRC(C), $\zeta=.02$ & 53.17 & 0.880 \\
Gaussian noise & CBDL-A3 & 90.30 & 0.788 \\
Gaussian noise & CDEC-A3 & 47.92 & 0.934 \\
Gaussian noise & IDEC-CRC(C), $\zeta=.02$ & 50.48 & 0.841 \\
Central occlusion & CBDL-A3 & 71.53 & 0.756 \\
Central occlusion & CDEC-A3 & 74.82 & 0.814 \\
Central occlusion & IDEC-CRC(C), $\zeta=.02$ & 69.07 & 0.794 \\
\bottomrule
\end{tabular}
\caption{Artificial transformations of 2000 preselected real test images. AUROC compares corrupted and unmodified source images with larger scores denoting corruption. These are corruption diagnostics, not natural novelty datasets or ground truth AU/EU measurements.}
\label{tab:mnist-corruption}
\end{table}

The CDEC-reference expected risk IDEC model also exhibits a substantial
increase in reference discrepancy under Gaussian noise, with a mean
value of $0.347$. Its interval width was calibrated using unmodified
inputs. The guarantee in Theorem~\ref{thm:radius-crc} requires
exchangeability between calibration and future inputs, so calibration
on the clean population does not establish the same guarantee for
these corrupted inputs.
These experiments concern specified artificial transformations of
real test images. 

\end{document}